\documentclass{article}
\usepackage[utf8]{inputenc}
\usepackage{mathtools}
\usepackage[a4paper,top=3cm,bottom=2cm,left=2.3cm,right=2.3cm,marginparwidth=1.75cm]{geometry}

\title{Tackling the Curse of Dimensionality with Physics-Informed Neural Networks}

\date{}

\usepackage{etoolbox}
\makeatletter
\patchcmd{\@maketitle}{\LARGE}{\fontsize{15}{20}\selectfont}{}{}
\makeatother

\usepackage{graphicx}
\usepackage{verbatim}
\usepackage{hyperref}       
\usepackage{url}            
\usepackage{booktabs}       
\usepackage{amsfonts}       
\usepackage{nicefrac}       
\usepackage{microtype}      
\usepackage{amsmath}
\usepackage{multirow}
\usepackage{amsthm}
\newtheorem{theorem}{Theorem}[section]
\newtheorem{remark}{Remark}
\newtheorem{lemma}{Lemma}[section]
\newtheorem*{lemma*}{Lemma}
\newtheorem*{theorem*}{Theorem}
\newtheorem*{assumption*}{Assumption}
\newtheorem{assumption}{Assumption}[section]
\newtheorem{definition}{Definition}[section]

\usepackage{subfigure}
\usepackage{times}
\usepackage{color}
\usepackage{xcolor}
\usepackage{algorithm}
\usepackage{algpseudocode}

\newcommand{\bx}{\boldsymbol{x}}

\begin{document}

\author{Zheyuan Hu\thanks{Department of Computer Science, School of Computing, National University of Singapore, Singapore, 119077. Email: \href{mailto:e0792494@u.nus.edu}{e0792494@u.nus.edu} (Z Hu),\href{mailto:kenji@nus.edu.sg}{kenji@nus.edu.sg} (KK)}
\and Khemraj Shukla\thanks{Division of Applied Mathematics, Brown University, Providence, RI 02912, USA (\href{mailto:khemraj\_shukla@brown.edu}{khemraj\_shukla@brown.edu}, \href{mailto:george\_karniadakis@brown.edu}{george\_karniadakis@brown.edu})}
\and George Em Karniadakis\footnotemark[2] \and  \linebreak Kenji Kawaguchi\footnotemark[1]}

\maketitle
\begin{abstract}
The curse-of-dimensionality taxes computational resources heavily with exponentially increasing computational cost as the dimension increases. This poses great challenges in solving high-dimensional partial differential equations (PDEs), as Richard E. Bellman first pointed out over 60 years ago. While there has been some recent success in solving numerical PDEs in high dimensions, such computations are prohibitively expensive, and true scaling of general nonlinear PDEs to high dimensions has never been achieved. We develop a new method of scaling up physics-informed neural networks (PINNs) to solve arbitrary high-dimensional PDEs. The new method, called Stochastic Dimension Gradient Descent (SDGD), decomposes a gradient of PDEs' and PINNs' residual into pieces corresponding to different dimensions and randomly samples a subset of these dimensional pieces in each iteration of training PINNs. We prove theoretically the convergence and other desired properties of the proposed method. We demonstrate in various diverse tests that the proposed method can solve many notoriously hard high-dimensional PDEs, including the Hamilton-Jacobi-Bellman (HJB) and the Schr\"{o}dinger equations in tens of thousands of dimensions very fast on a single GPU using the PINNs mesh-free approach. Notably, we solve nonlinear PDEs with nontrivial, anisotropic, and inseparable solutions in less than one hour for 1,000 dimensions and in 12 hours for 100,000 dimensions on a single GPU using SDGD with PINNs. Since SDGD is a general training methodology of PINNs, it can be applied to any current and future variants of PINNs to scale them up for arbitrary high-dimensional PDEs.
\end{abstract}

\section{Introduction}
The curse-of-dimensionality (CoD) refers to the computational and memory challenges when dealing with high-dimensional problems that do not exist in low-dimensional settings. The term was first introduced in 1957 by Richard E. Bellman \cite{bellman1966dynamic,hammer1962adaptive}, who was working on dynamic programming, to describe the exponentially increasing costs due to high dimensions. Later, the concept was adopted in various areas of science, including numerical partial differential equations (PDEs), combinatorics, machine learning, probability, stochastic analysis, and data mining. 
In the context of numerically solving PDEs, an increase in the dimensionality of the PDE's independent variables tends to lead to a corresponding exponential rise in computational costs needed for obtaining an accurate solution. This poses a considerable challenge for all existing algorithms.

Physics-Informed Neural Networks (PINNs) are highly practical in solving PDE problems \cite{raissi2019physics}. Compared to other methods, PINNs can solve various PDEs in general form, enabling mesh-free solutions and handling complex geometries. Additionally, PINNs leverage the interpolation capability of neural networks to provide accurate predictions throughout the domain, unlike other methods for high-dimensional PDEs, which can only compute the value of PDEs at a single point or are restricted to certain PDE types \cite{beck2021deep, han2018solving, raissi2018forward}. Due to neural networks' flexibility and expressiveness, in principle, the set of functions representable by PINNs have the ability to approximate high-dimensional PDE solutions. However, PINNs may also fail in tackling CoD due to insufficient memory and slow convergence when dealing with high-dimensional PDEs. For example, when the dimension of PDE is very high, even just using one collocation point can lead to an insufficient memory error due to the massive high-dimensional derivatives in the PDE. Unfortunately, the aforementioned high-dimensional and high-order cases often occur for very useful PDEs, such as the Hamilton-Jacobi-Bellman (HJB) equation in stochastic optimal control, the Fokker-Planck equation in stochastic analysis and high-dimensional probability, and the Black-Scholes equation in mathematical finance, etc. These PDE examples pose a grand challenge to the application of PINNs in effectively tackling real-world large-scale problems.

To scale up PINNs for arbitrary high-dimensional PDEs, we propose a training method of PINNs by decomposing and sampling gradients of PDE's and PINN's residual, which we name {\em stochastic dimension gradient descent} (SDGD). It accelerates the training of PINNs while significantly reducing the memory cost, thereby realizing the training of any high-dimensional PDEs with acceleration. Specifically, we scale up and speed up PINNs by decomposing the computational and memory bottleneck, namely the gradient of the residual loss that contains massive PDE terms along numerous dimensions. We propose a novel decomposition of the gradient of the residual loss into each piece corresponding to each dimension of PDEs. Then, at each training iteration, we sample a subset of these dimensional pieces to optimize PINNs. Although we use only a subset per iteration, this sampling process is ensured to be an unbiased estimator of the full gradient of all dimensions, which is then used to guarantee convergence for all dimensions. 

SDGD enables more efficient parallel computations, fully leveraging multi-GPU computing to scale up and expedite the speed of PINNs. Parallelizing the computation of mini-batches of different dimensional pieces across multiple devices can expedite convergence. This is akin to traditional SGD over data points, which allows simultaneous computation on different machines. SDGD also accommodates the use of gradient accumulation on resource-limited machines to enable larger batch sizes and reduce gradient variance. It involves accumulating gradients from multiple batches to form a larger batch size for stochastic gradients, reducing the variance of each stochastic gradient on a single GPU. Due to the refined and general partitioning offered by SDGD that samples both collocation points and PDE dimensions, our gradient accumulation can be executed on devices with limited resources, ideal for edge computing. 

We theoretically prove that the stochastic gradients generated by SDGD are unbiased. Based on this, we also provide convergence guarantees for our method. Additionally, our theoretical analysis shows that under the same batch size and memory cost, proper batch sizes for PDE terms/dimensions and residual points can minimize gradient variance and accelerate convergence, an outcome not possible solely with conventional SGD over points. SDGD extends and generalizes traditional SGD over points, providing stable, low-variance stochastic gradients while enabling smaller batch sizes and reduced memory usage.

We conduct extensive experiments on several high-dimensional PDEs. We investigate the relationship between SDGD and vanilla SGD over collocation points, where the latter is stable and widely adopted to reduce memory cost in previous work. We vary SDGD's batch sizes of PDE terms/dimensions and collocation points to study the stability and convergence speed of SDGD under certain memory budgets. Experimental results show that our proposed SDGD is as stable as SGD over points on nonlinear high-dimensional PDEs and can accelerate convergence. We also demonstrate the guideline for selecting the batch sizes of PDE terms/dimensions and collocation points in SDGD.

Furthermore, we showcase large-scale PDEs where traditional PINN training directly fails due to an out-of-memory (OOM) error. Concretely, with the further speed-up method via sampling PDE terms/dimensions in both forward and backward passes, we can train PINNs on nontrivial nonlinear PDEs in 100,000 dimensions in 12 hours of training time, while the traditional PINN training methods exceed the GPU memory limit directly. An in-depth comparison between vanilla PINN training and SDGD is also provided.

After validating the effectiveness of SDGD compared to SGD over collocation points, we compare the algorithms with other methods, e.g., \cite{beck2021deep, han2018solving, raissi2018forward}. Our algorithm outperforms other methods on multiple nonlinear parabolic PDEs. Additionally, our method enables mesh-free training and prediction across the entire domain, while other methods typically only provide predictions at a single point. Furthermore, we combine our algorithm with adversarial training, notorious for its slowness in high-dimensional machine learning, and our algorithm significantly accelerates it. We also demonstrate the generality of our method by applying it to the Schr\"{o}dinger equation, which has broad connections with quantum physics. Overall, our method provides a paradigm shift for performing large-scale PINN training. 

The rest of this paper is arranged as follows. We present related work in Section 2, our main algorithm for accelerating and scaling up PINNs in Section 3, theoretical analysis of our algorithms' convergence in Section 4, numerical experiments in Section 5, and conclusion in Section 6.
In the Appendices, we include more examples and details of our method and prove the theories.

\section{Related Work}
\subsection{Physics-Informed Machine Learning}
This paper is based on the concept of Physics-Informed Machine Learning \cite{karniadakis2021physics}. Specifically, PINNs \cite{raissi2019physics} utilize neural networks as surrogate solutions for PDEs and optimize the boundary loss and residual loss to approximate the PDEs' solutions. On the other hand, DeepONet \cite{lu2019deeponet} leverages the universal approximation theorem of infinite-dimensional operators and directly fits the PDE operators in a data-driven manner. It can also incorporate physical information by using residual loss. Therefore, our algorithms can also be flexibly combined with the training of physics-informed DeepONet \cite{goswami2022physicsinformed} and other variants such as Fourier neural operators \cite{li2021fourier,JMLR:v24:21-1524} and its physics-informed version \cite{li2021physics}. Since their inception, PINNs have succeeded in solving numerous practical problems in computational science, e.g., nonlinear PDEs \cite{raissi2018hidden}, solid mechanics \cite{haghighat2021physics}, uncertainty quantification \cite{yang2019adversarial}, inverse water waves problems \cite {jagtap2022deep}, and fluid mechanics \cite{cai2021physics}, to name just a few. Neural operators also have shown their great potential in solving stiff chemical mechanics \cite{goswami2023learning}, shape optimization \cite{shukla2023deep}, and materials science problems \cite{goswami2022physics}.

The success of PINNs has also attracted considerable attention in theoretical analysis. In the study by Luo et al. \cite{Luo2020TwoLayerNN}, the authors delved into exploring PINNs' prior and posterior generalization bounds. Additionally, they utilized the neural tangent kernel to demonstrate the global convergence of PINNs. Mishra et al. \cite{mishra2020estimates} derived an estimate for the generalization error of PINNs, considering both the training error and the number of training samples. In a different approach, Shin et al. \cite{shin2020convergence} adapted the Schauder approach and the maximum principle to provide insights into the convergence behavior of the minimizer as the number of training samples tends towards infinity. They demonstrated that the minimizer converges to the solution in both $C^0$ and $H^1$. Furthermore, Lu et al. \cite{lu2021priori} employed the Barron space within the framework of two-layer neural networks to conduct a prior analysis on PINN with the softplus activation function. This analysis is made possible by drawing parallels between the softplus and ReLU activation functions. More recently, Hu et al. \cite{hu2021extended} employed the Rademacher complexity concept to measure the generalization error for both PINNs and the extended PINNs variant (XPINNs \cite{jagtap2020extended} and APINNs \cite{hu2022augmented}).

\subsection{Machine Learning PDEs Solvers in High-Dimensions}
Numerous attempts have been made to tackle high-dimensional PDEs by deep learning to overcome the curse of dimensionality.
In the PINN literature, \cite{wang20222} showed that to learn a high-dimensional HJB equation, the $L^\infty$ loss is required. The authors of \cite{he2023learning} proposed to parameterize PINNs by the Gaussian smoothed model and to optimize PINNs without back-propagation by Stein's identity, avoiding the vast amount of differentiation in high-dimensional PDE operators to accelerate the convergence of PINNs. Separable PINN \cite{cho2022separable} considers a per-axis sampling of residual points in high-dimensional spaces, thereby reducing the computational cost and enlarging the batch size of residual points in PINN via tensor product. However, this method \cite{cho2022separable} focuses on acceleration on only modest (less than 4) dimensions and is designed mainly for increasing the number of collocation points in 3D PDEs, whose separation of sampling points becomes intractable in high-dimensional spaces. 
Han et al. \cite{han2018solving, han2017deep} proposed the DeepBSDE solver for high-dimensional PDEs based on the classical BSDE method, and they used deep learning to approximate unknown functions. Extensions of the DeepBSDE method were presented in \cite{beck2019machine, chan2019machine,henry2017deep,hure2020deep,ji2020three}. Becker et al. \cite{becker2021solving} solved high-dimensional optimal stopping problems via deep learning. The authors of \cite{beck2021deep} combined splitting methods and deep learning to solve high-dimensional PDEs. Raissi \cite{raissi2018forward} leveraged the relationship between high-dimensional PDEs and stochastic processes whose trajectories provide supervision for deep learning. 
Despite the effectiveness of the methods in \cite{beck2021deep, han2018solving, raissi2018forward, han2017deep}, they can only be applied to a certain restricted class of PDEs.
Wang et al. \cite{wang2022tensor, wang2022solving} proposed tensor neural networks with efficient numerical integration and separable structures for solving high-dimensional Schr\"{o}dinger equations in quantum physics. Zhang et al. \cite{zhang2020learning} proposed using PINNs for solving stochastic differential equations by representing their solutions via spectral dynamically orthogonal and bi-orthogonal methods. Zang et al. \cite{zang2020weak} proposed a weak adversarial network that solves PDEs using the weak formulation.
The deep Galerkin method (DGM) \cite{sirignano2018dgm} optimizes networks to satisfy the high-dimensional PDE operator, where the derivatives are estimated via Monte Carlo. The deep Ritz method (DRM) \cite{Weinan2017TheDR} considers solving high-dimensional variation PDE problems.

In the numerical PDE literature, there have been attempts to scale numerical methods to high dimensions, e.g., proper generalized decomposition (PGD) \cite{chinesta2011short}, multi-fidelity information fusion algorithm \cite{perdikaris2016multifidelity}, and ANOVA \cite{martens2020neural,zhang2012error}. Darbon and Osher \cite{darbon2016algorithms} proposed a fast algorithm for solving high-dimensional Hamilton-Jacobi equations if the Hamiltonian is convex and positively homogeneous of degree one. The multilevel Picard method \cite{beck2020overcoming,beck2020overcoming_ac,becker2020numerical,hutzenthaler2020overcoming, hutzenthaler2021multilevel} is another approach for approximating solutions of high-dimensional parabolic PDEs, which reformulates the PDE problem as a stochastic fixed point
equation, which is then solved by multilevel and nonlinear Monte-Carlo. Similar to Beck et al. \cite{beck2021deep} and Han et al. \cite{han2018solving}, the multilevel Picard method can only output the solution's value on one test point and can only be applied to a certain restricted class of parabolic PDEs.

\subsection{Stochastic Gradient Descent}
This paper accelerates and scales up PINNs based on the idea of stochastic gradient descent (SGD). In particular, we propose the aforementioned SDGD method. SGD is a standard approach in machine learning for handling large-scale data. As an iterative optimization algorithm, it updates the model's parameters by computing the gradients on a small randomly selected subset of training examples, known as mini-batches. This randomness introduces stochasticity, hence enabling faster convergence and efficient utilization of large datasets, making SGD a popular choice for training deep learning models.
Among the theoretical guarantees of SGD, the works in \cite{fehrman2020convergence, lei2019stochastic, mertikopoulos2020almost} are remarkable milestones. Specifically, the condition presented in \cite{lei2019stochastic} stands as a minimum requirement within the general optimization literature, specifically about the Lipschitz continuity of the gradient estimator. Their proof can also be extended to establish convergence for our particular case, necessitating the demonstration of an upper bound on the Lipschitz constant of our gradient estimator. On the other hand, \cite{fehrman2020convergence, mertikopoulos2020almost}  adopted a classical condition that entails a finite upper bound on the stochastic gradient's variance. In this case, it suffices to compute an upper bound for the variance term. While \cite{lei2019stochastic} assumed a more relaxed condition, the conclusion is comparably weaker, demonstrating proven convergence but with a notably inferior convergence rate. 
In the theoretical analysis presented in section 4 of this paper, we will utilize the aforementioned tools to provide convergence guarantees for our SDGD, emphasizing SDGD's stochastic gradient variance and PINNs convergence.

\section{Method}\label{sec:method}
\subsection{Physics-Informed Neural Networks (PINNs)}
In this paper, we focus on solving the following partial differential equations (PDEs) defined on a domain $\Omega \subset \mathbb{R}^d$:
\begin{equation}\label{eq:PDE}
\begin{aligned}
\mathcal{L}u(\bx)=R(\bx) \ \text{in}\ \Omega, \qquad
\mathcal{B}u(\bx)=B(\bx) \ \text{on}\ \Gamma,
\end{aligned}
\end{equation}
where $\mathcal{L}$ and $\mathcal{B}$ are the differential operators for the residual in $\Omega$ and for the boundary/initial condition on $\Gamma$.
PINN \cite{raissi2019physics} is a neural network-based PDE solver via minimizing the following boundary $\mathcal{L}_b(\theta)$ and residual loss $\mathcal{L}_r(\theta)$ functions.
\begin{equation}
\begin{aligned}
\mathcal{L}(\theta) &= \lambda_b \mathcal{L}_b(\theta) + \lambda_r \mathcal{L}_r(\theta)=\frac{\lambda_b}{n_b}\sum_{i=1}^{n_b} {|\mathcal{B}u_{\theta}(\bx_{b,i})-B(\bx_{b,i})|}^2 + \frac{\lambda_r}{n_r}\sum_{i=1}^{n_r} {|\mathcal{L}u_{\theta}(\bx_{r,i})-R(\bx_{r,i})|}^2.
\end{aligned}
\end{equation}
where $\lambda_b, \lambda_r > 0$ are the weights for balancing the losses. $n_b, n_r$ are the number of boundary points $\{\bx_{b,i}\}_{i=1}^{n_b} \subset \Gamma$ and residual points $\{\bx_{r,i}\}_{i=1}^{n_r} \subset \Omega$, respectively. $u_\theta$ is the neural network model parameterized by $\theta$.

\subsection{Methodology for High-Dimensional PDEs}
We first adopt the simple high-dimensional second-order Poisson's equation for illustration; then, we move to the general case covering a variety of high-dimensional PDEs.
\subsubsection{Introductory Case of the High-Dimensional Poisson's Equation}
We consider a simple high-dimensional second-order Poisson's equation for illustration:
\begin{equation}
\Delta u(\bx) = \sum_{i=1}^d \frac{\partial^2}{\partial\bx_i^2}u(\bx) = R(\bx), \quad\bx \in \Omega \subset \mathbb{R}^d,
\end{equation}
where $u$ is the unknown exact solution, and $u_\theta$ is our PINN model parameterized by $\theta$. The memory scales quadratically as $d$ increases due to the growing network size at the first input layer of $u_\theta$ and the increasing number of second-order derivatives. So, for extremely high-dimensional PDEs containing many second-order terms, using only one collocation point can lead to insufficient memory, whose memory cost cannot be further reduced in the traditional SGD over collocation points scenario. The huge memory cost is also incurred by the large input size at the first layer of $u_\theta$, which is the same as the PDE dimensionality.

However, the memory problem is solvable by inspecting the residual loss on the collocation point $\bx$:
\begin{equation}
\ell(\theta) = \frac{1}{2}\left(\sum_{i=1}^d \frac{\partial^2}{\partial\bx_i^2}u_\theta(\bx) - R(\bx)\right)^2,
\end{equation}
The gradient with respect to the model parameters $\theta$ for training the PINN is
\begin{equation}
\begin{aligned}
\text{grad}(\theta) := \frac{\partial\ell(\theta)}{\partial \theta} &= \textcolor{blue}{\left(\sum_{i=1}^d \frac{\partial^2}{\partial\bx_i^2}u_\theta(\bx) - R(\bx)\right)}\left(\sum_{i=1}^d \textcolor{red}{\frac{\partial}{\partial\theta}\frac{\partial^2}{\partial\bx_i^2}u_\theta(\bx)}\right).
\end{aligned}
\end{equation}
We are only differentiating with respect to parameters $\theta$ on the $d$ PDE terms $\textcolor{red}{\frac{\partial^2}{\partial\bx_i^2}u_\theta(\bx)}$, which is the memory bottleneck in the backward pass since the shape of the gradient is proportional to both the PDE dimension and the parameter count of the PINN. In contrast, the first part $\textcolor{blue}{\left(\sum_{i=1}^d \frac{\partial^2}{\partial\bx_i^2}u_\theta(\bx) - R(\bx)\right)}$ is a scalar, which can be precomputed and detached from the GPU since it is not involved in the backpropagation for $\theta$. Since the full gradient grad$(\theta)$ is the sum of $d$ independent terms, we can sample several terms for stochastic gradient descent (SGD) using the sampled unbiased gradient estimator. Concretely, our algorithm can be summarized as follows:
\begin{enumerate}
\item Choose random indices $I \subset \{1,2,\cdots,d\}$ where $|I|$ is the cardinality of the set $I$, which is the batch size over PDE terms, where we can set $|I| \ll d$ to minimize memory cost. 
\item For $i=1,2,\cdots,d$, compute $\frac{\partial^2}{\partial\bx_i^2}u_\theta(\bx)$.
If $i \in I$, then keep the gradient with respect to $\theta$, else we detach it from the GPU to save memory. After detachment, the term will not be involved in the costly backpropagation, and its gradient with respect to $\theta$ will not be computed.
\item Compute the unbiased stochastic gradient used to update the model
\begin{equation}
\text{grad}_I(\theta) = \textcolor{red}{\frac{d}{|I|}}\left(\sum_{i=1}^d \frac{\partial^2}{\partial\bx_i^2}u_\theta(\bx) - R(\bx)\right)\left(\textcolor{red}{\sum_{i \in I}} \frac{\partial^2}{\partial\bx_i^2}\frac{\partial}{\partial\theta}u_\theta(\bx)\right).
\end{equation}
\item If not converged, go to 1.
\end{enumerate}
Our algorithm enjoys the following good properties and extensions:
\begin{itemize}
\item Low memory cost: Since the main cost is from the backward pass, and we are only backpropagating over terms with $i \in I$, the cost is the same as the corresponding $|I|$-dimensional PDE.
\item Unbiased stochastic gradient: Our gradient is an {unbiased} estimator of the true full batch gradient, i.e., 
$
\mathbb{E}_I \left[\text{grad}_I(\theta)\right] = \text{grad}(\theta),
$
so that modern SGD accelerators such as Adam \cite{kingma2014adam} can be adopted.
\item Accumulate gradient for full batch GD: For the full GD that exceeds the memory, we can select non-overlapping index sets $\{I_k\}_{k \in K}$ such that $\cup_k I_k = \{1,2,\cdots,d\}$, then we can combine these minibatch gradients to get the full gradient,
$
\frac{1}{|K|}\sum_{k} \text{grad}_{I_k}(\theta) = \text{grad}(\theta).
$
This baseline process is memory-efficient but time-consuming. It is memory-efficient since it divides the entire computationally intensive gradient into the sum of several stochastic gradients whose computations are memory-efficient. This operation is conducted one by one using the ``For loop from $i=1$ to $d$". But it is rather time-consuming because the ``For loop" cannot be parallelized.
\end{itemize}
A common practice in PINNs is to sample stochastic collocation points to reduce the batch size of points, which is a common way to reduce memory costs in previous methods and even the entire field of deep learning. Next, we compare SDGD with SGD over collocation points:
\begin{itemize}
\item SGD on collocation points has a minimum batch size of 1 point plus the $d$-dimensional equation, i.e., the minimal cost of vanilla SGD over train collocation/residual points solely is choosing one residual point $\bx$ and backpropagate model parameters $\theta$ through all the $d$ dimensions:
\begin{equation}
\begin{aligned}
\text{grad}(\theta) = \left(\sum_{i=1}^{d}\frac{\partial^2}{\partial\bx_i^2} u_\theta(\bx) - R(\bx)\right) \underbrace{\left( {\sum_{i = 1}^{d}}\frac{\partial^2}{\partial\bx_i^2}\frac{\partial}{\partial\theta}u_\theta(\bx)\right)}_{\text{backproprogation over } d \text{ dimensions}}
\end{aligned}
\end{equation}
\item SDGD can be combined with SGD on points, so its minimal batch size is 1 point plus 1D equation, i.e., selecting a point and then backpropagating over only one dimension with the unbiased stochastic gradient as follows:
\begin{equation}
\begin{aligned}
\text{grad}_{\{i\}}(\theta) = d\left(\sum_{i=1}^{d}\frac{\partial^2}{\partial\bx_i^2} u_\theta(\bx) - R(\bx)\right) \underbrace{\left( \frac{\partial^2}{\partial\bx_i^2}\frac{\partial}{\partial\theta}u_\theta(\bx)\right)}_{\text{backproprogation over } 1 \text{ dimension}}
\end{aligned}
\end{equation}
\item Thus, our method SDGD can solve high-dimensional PDEs arbitrarily since its minimal computational cost will grow much more slowly than conventional SGD as the dimension $d$ increases.
Empirically, it is interesting to see how the two SGDs affect convergence. In particular, $B$-point + $D$-term with the same $B \times D$ quantity has the same memory cost, e.g., 50-point-100-terms and 500-point-10-terms.
\end{itemize}

\subsubsection{General Case}
Now, we move from the simple Poisson's equation to more general cases.
We are basically performing the following decomposition on the PDE differential operator/PDE terms:
\begin{equation}
\mathcal{L}u = \sum_{i=1}^{N_{\mathcal{L}}}\mathcal{L}_i u,
\end{equation}
where $\mathcal{L}_i$ are the $N_{\mathcal{L}}$ decomposed PDE terms. We have the following examples:
\begin{itemize}
\item In the $d$-dimensional Poisson cases, $N_{\mathcal{L}} = d$ and $\mathcal{L}_i u = \frac{\partial^2}{\partial \bx_i^2} u(\bx)$. This observation also highlights the significance of decomposition methods for high-dimensional Laplacian operators. The computational bottleneck in many high-dimensional second-order PDEs often lies in the high-dimensional Laplacian operator.
\item Consider the $d$-dimensional Fokker–Planck (FP) equations, which are ubiquitous in science and engineering. It is originally derived from statistical mechanics to describe the evolution of stochastic differential equations (SDEs). Recently, a real-world application of SDE and its associated FP equation is for generative modeling \cite{song2021scorebased}. Furthermore, it also generalizes the Black-Scholes equation in mathematical finance, the Hamilton-Jacobi-Bellman equation in optimal control, and Schr\"{o}dinger equation in quantum physics.
\begin{equation}
\frac{\partial u(\boldsymbol{x},t)}{\partial t} = -\sum_{i=1}^{d} \frac{\partial}{\partial \bx_i} \left[ F_i(\boldsymbol{x}) u(\bx,t) \right] + \frac{1}{2} \sum_{i,j=1}^{d} \frac{\partial^2}{\partial \bx_i \partial \bx_j} \left[ D_{ij}(\bx) u(\bx,t) \right],
\end{equation}
where $\bx = (\bx_1, \bx_2, \dots, \bx_n)$ is a $d$-dimensional vector, $u(\bx,t)$ is the probability density function of the stochastic process $\bx(t)$, $F_i(\boldsymbol{x})$ is the $i$-th component of the drift coefficient vector $\boldsymbol{F}(\bx)$, and $D_{ij}(\bx)$ is the $(i,j)$-th component of the diffusion coefficient matrix $\boldsymbol{D}(\bx)$. 
Then, our decomposition over the PDE terms is
\begin{equation}
\begin{aligned}
\mathcal{L}u &= \frac{\partial u(\boldsymbol{x},t)}{\partial t} +\sum_{i=1}^{d} \frac{\partial}{\partial \bx_i} \left[ F_i(\boldsymbol{x}) u(\bx,t) \right] -\frac{1}{2} \sum_{i,j=1}^{d} \frac{\partial^2}{\partial \bx_i \partial \bx_j} \left[ D_{ij}(\bx) u(\bx,t) \right]\\
&= \sum_{i,j=1}^d  \left\{\frac{1}{d^2}\frac{\partial u(\boldsymbol{x},t)}{\partial t} + \frac{1}{d}\frac{\partial}{\partial \bx_i} \left[ F_i(\boldsymbol{x}) u(\bx,t) \right] - \frac{1}{2} \frac{\partial^2}{\partial \bx_i \partial \bx_j} \left[ D_{ij}(\bx) u(\bx,t) \right]\right\}= \sum_{i,j=1}^d \mathcal{L}_{ij}u(\bx),
\end{aligned}
\end{equation}
where we denoted 
$
\mathcal{L}_{ij}u(\bx) = \frac{1}{d^2}\frac{\partial u(\boldsymbol{x},t)}{\partial t} + \frac{1}{d}\frac{\partial}{\partial \bx_i} \left[ F_i(\boldsymbol{x}) u(\bx,t) \right] - \frac{1}{2} \frac{\partial^2}{\partial \bx_i \partial \bx_j} \left[ D_{ij}(\bx) u(\bx,t) \right].
$
\item The $d$-dimensional bi-harmonic equation:
$
\Delta^2 u(\bx) = 0, \bx \in \Omega \subset \mathbb{R}^d,
$
where $\Delta = \sum_{i=1}^{d}\frac{\partial^2}{\partial \bx_i^2}$ is the Laplacian operator. For this high-dimensional fourth-order PDE, we can decompose as follows for the PDE operator/terms:
\begin{equation}
\mathcal{L}u = \Delta^2 u(\bx) = \sum_{i,j=1}^d \frac{\partial^4}{\partial \bx_i^2\partial \bx_j^2} u(\bx) = \sum_{i,j=1}^d \mathcal{L}_{ij}u(\bx),
\quad \text{where }
\mathcal{L}_{ij}u(\bx) = \frac{\partial^4}{\partial \bx_i^2\partial \bx_j^2} u(\bx).
\end{equation}
\item For other high-dimensional PDEs, their complexity arises from the dimensionality itself. Therefore, it is always possible to find decomposition methods that can alleviate this complexity.
\end{itemize}

We go back to the algorithm after introducing these illustrative examples; the computational and memory bottleneck is the residual loss:
\begin{equation}
\ell(\theta) =\frac{1}{2} \left(\mathcal{L}u(\bx) - R(\bx)\right)^2.
\end{equation}
The gradient with respect to the model parameters $\theta$ for training the PINN is
\begin{equation}\label{eq:full_pinn_grad}
\begin{aligned}
\text{grad}(\theta) = \frac{\partial\ell(\theta)}{\partial \theta} &= \textcolor{blue}{\left(\sum_{i=1}^{N_{\mathcal{L}}} \mathcal{L}_iu_\theta(\bx) - R(\bx)\right)}\left(\sum_{i=1}^{N_{\mathcal{L}} }\textcolor{red}{\frac{\partial}{\partial\theta} \mathcal{L}_iu_\theta(\bx)}\right).
\end{aligned}
\end{equation}
Subsequently, we are sampling the red part to reduce memory cost and to accelerate training. Our algorithm can be summarized in Algorithm \ref{algo:1}.
\begin{algorithm}
\caption{Training Algorithm for scaling-up by sampling the dimension in the backward pass.}
\begin{algorithmic}[1]
\While{NOT CONVERGED}
\State Choose random indices $I \subset \{1,2,\cdots, N_{\mathcal{L}}\}$ where $|I|$ is the cardinality of the set $I$, which is the batch size over PDE terms, where we can set $|I| \ll N_{\mathcal{L}}$ to minimize memory cost. 
\For{$i \in \{1,2,\cdots, N_{\mathcal{L}}\}$}
    \State compute $\mathcal{L}_iu_\theta(\bx)$.
If $i \in I$, then keep the gradient with respect to $\theta$; else, we detach it from the GPU to save memory. After detachment, the term will not be involved in the costly backpropagation, and its gradient with respect to $\theta$ will not be computed, which saves GPU memory costs.
\EndFor
\State Compute the unbiased stochastic gradient used to update the model as below:
\begin{equation}\label{eq:sdgd_grad_algo1}
\text{grad}_I(\theta) = {\frac{N_{\mathcal{L}}}{|I|}}\left(\sum_{i=1}^{N_\mathcal{L}}\mathcal{L}_iu_\theta(\bx) - R(\bx)\right)\left({\sum_{i \in I}} \frac{\partial}{\partial\theta}\mathcal{L}_iu_\theta(\bx)\right).
\end{equation}
\EndWhile
\end{algorithmic}
\label{algo:1}
\end{algorithm}

\subsection{Extensions: Gradient Accumulation and Parallel Computing}\label{sec:GAPC}
This subsection introduces two straightforward extensions of our Algorithm \ref{algo:1} for further speed-up and scale-up.

\textbf{Gradient Accumulation}. Since our method involves SDGD in large-scale problems, we must reduce the batch size to fit it within the available GPU memory. However, a very small batch size can result in significant gradient variance. In such cases, gradient accumulation is a promising idea. Specifically, gradient accumulation involves sampling different index sets $I_1, I_2, \cdots, I_n$, computing the corresponding gradients $\text{grad}_{I_k}(\theta)$ for $k = 1,2,\cdots,n$, and then averaging them as the final unbiased stochastic gradient for one-step optimization to effectively increase the batch size. The entire process can be implemented on a single GPU, and we only need to be mindful of the tradeoff between computation time and gradient variance.

\textbf{Parallel Computing}. Just as parallel computing can be used in machine learning to increase batch size and accelerate convergence in SGD, our new SDGD also supports parallel computing. Recall that the stochastic gradient by sampling the PDE terms randomly is given by $\text{grad}_I(\theta)$ in equation (\ref{eq:sdgd_grad_algo1}).
We can compute the above gradient for different index sets $I_1, I_2, \cdots, I_n$ on various machines simultaneously and accumulate them to form a larger-batch stochastic gradient.

\subsection{Further Speed-up via Simultaneous Sampling in Forward and Backward Passes}
In this section, we discuss how to accelerate further large-scale PINN training based on our proposed memory reduction methods to make it both fast and memory efficient.

Although the proposed methods in the previous Algorithm \ref{algo:1} can significantly reduce memory cost and use SDGD to obtain some acceleration through gradient randomness, the speed is still slow for particularly large-scale problems because we need to calculate each term $\{\mathcal{L}_iu(\bx)\}_{i=1}^{N_{\mathcal{L}}}$ in the forward pass one by one, and we are only omitting some these terms in the backward pass to scale up. For example, in extremely high-dimensional cases, we may need to calculate thousands of second-order derivatives of PINN, and the calculation speed is clearly unacceptable.

To overcome this bottleneck, we can perform the same unbiased sampling on the forward pass to accelerate it while ensuring that the entire gradient is unbiased. I.e., we only select some indices for calculation in the forward pass and select another set of indices for the backward pass, combining them into a very cheap yet unbiased stochastic gradient.

Mathematically, consider the full gradient $\text{grad}(\theta)$ again given in equation (\ref{eq:full_pinn_grad}).
We choose two random and independent indices sets $I, J \subset \{1,2,\cdots, N_{\mathcal{L}}\}$ for the sampling of the backward and the forward passes, respectively. The corresponding stochastic gradient is:
\begin{equation}\label{eq:sdgd_grad_algo2}
\begin{aligned}
\text{grad}_{I,J}(\theta) =\frac{N_{\mathcal{L}}}{|I|}\left(\left(\frac{N_{\mathcal{L}}}{|J|}\sum_{j \in J} \mathcal{L}_ju_\theta(\bx)\right) - R(\bx)\right)\left(\sum_{i\in I}\frac{\partial}{\partial\theta} \mathcal{L}_iu_\theta(\bx)\right).
\end{aligned}
\end{equation}
Since $I, J$ are independent, the gradient estimator is unbiased:
$
\mathbb{E}_{I,J}\text{grad}_{I,J}(\theta) = \text{grad}(\theta).
$
Our algorithm can be summarized in Algorithm \ref{algo:2}.

\begin{algorithm}
\caption{Training Algorithm for scaling-up and further speeding-up by sampling forward and backward passes.}
\begin{algorithmic}[1]
\While{NOT CONVERGED}
\State Choose random indices $I, J \subset \{1,2,\cdots, N_{\mathcal{L}}\}$, where we can set $|I| \ll N_{\mathcal{L}}$ to minimize memory cost and $|J| \ll N_{\mathcal{L}}$ to further speed up.
\For{$i \in I$}
    \State compute $\mathcal{L}_iu_\theta(\bx)$ and keep the gradient with respect to $\theta$.
\EndFor
\For{$j \in J$}
    \State compute $\mathcal{L}_ju_\theta(\bx)$ and detach it to save memory cost.
\EndFor
\State Compute the unbiased stochastic gradient used to update the model given in equation (\ref{eq:sdgd_grad_algo2}).
\EndWhile
\end{algorithmic}
\label{algo:2}
\end{algorithm}

\textbf{Trading off Speed with Gradient Variance}. Obviously, since we introduce more randomness, the gradient variance of this method may be larger, and the convergence may be sub-optimal. However, at the initialization stage of the problem, an imprecise gradient is sufficient to make the loss drop significantly, which is the tradeoff between convergence quality and convergence speed. We will demonstrate in experiments that this method is particularly effective for extremely large-scale PINN training.

\begin{algorithm}
\caption{Training Algorithm for scaling-up and more speeding-up by sampling both the forward and backward passes only once.}
\begin{algorithmic}[1]
\While{NOT CONVERGED}
\State Choose random indices $I \subset \{1,2,\cdots, N_{\mathcal{L}}\}$, where we can set $|I| \ll N_{\mathcal{L}}$ to minimize memory cost and to further speed up.
\For{$i \in I$}
    \State compute $\mathcal{L}_iu_\theta(\bx)$ and keep the gradient with respect to $\theta$.
\EndFor
\State Compute the stochastic gradient used to update the model
\begin{equation*}
\begin{aligned}
\text{grad}_{I,I}(\theta) =\frac{N_{\mathcal{L}}}{|I|}\left(\left(\frac{N_{\mathcal{L}}}{|I|}\sum_{i \in I} \mathcal{L}_i u_\theta(\bx)\right) - R(\bx)\right)\left(\sum_{i\in I}\frac{\partial}{\partial\theta} \mathcal{L}_iu_\theta(\bx)\right).
\end{aligned}
\end{equation*}
\EndWhile
\end{algorithmic}
\label{algo:3}
\end{algorithm}

\subsection{Additional Speed-up via Sampling Only Once}
In the previous section, we introduced Algorithm \ref{algo:2}, which aims to accelerate SDGD by randomly sampling forward and backward passes. To ensure unbiased gradients, we independently sample the forward and backward passes. However, this leads to additional computational overhead. If we used the same sample for both the forward and backward passes, our computational cost would be even lower than Algorithm \ref{algo:2}, resulting in further acceleration. Nonetheless, the gradients obtained in this manner are biased due to the correlations between the forward and backward pass samples, and the batch size for sampling dimension $|I|$ controls the bias: a larger batch size $|I|$ leads to smaller bias. Fortunately, in practice, this approach often converges rapidly. We introduce this in Algorithm \ref{algo:3}.

The empirical success of Algorithm \ref{algo:3} can be attributed to the bias-speed tradeoff and its relatively low bias in practice. Algorithm \ref{algo:3} is the fastest due to the sampling only once. It can converge well thanks to the low bias in practice, as long as the batch size for sampling dimension $|I|$ is not too small, which will be experimentally demonstrated in Section \ref{sec:exp_bias_speed_tradeoff}.

\section{Theory}
In this section, we analyze the convergence of our proposed Algorithms \ref{algo:1} and \ref{algo:2}. 
\subsection{SDGD is Unbiased}
In previous sections, we have shown that the stochastic gradients generated by our proposed SDGD are unbiased:
\begin{theorem}\label{thm:unbiased}
The stochastic gradients $\text{grad}_I(\theta)$ and $\text{grad}_{I,J}(\theta)$ in our Algorithms \ref{algo:1} and \ref{algo:2}, respectively, parameterized by index sets $I, J$, are an unbiased estimator of the full-batch gradient $\text{grad}(\theta)$ using all PDE terms, i.e., the expected values of these estimators match that of the full-batch gradient, $\mathbb{E}_I[\text{grad}_I(\theta)]=\mathbb{E}_{I,J}[\text{grad}_{I,J}(\theta)] = \text{grad}(\theta)$.
\end{theorem}
\begin{proof}
We prove the theorem in \ref{appendix:unbiased}.
\end{proof}

\subsection{SDGD Reduces Gradient Variance}
In this subsection, we aim to show that SDGD can be regarded as another form of SGD over PDE terms, serving as a complement to the commonly used SGD over residual points. Specifically, $|B|$-point + $|I|$-term where $|B|, |I| \in \mathbb{Z}^+$ with the same $|B| \times |I|$ quantity has the same memory cost, e.g., 50-point-100-terms and 500-point-10-terms. In particular, we shall demonstrate that properly choosing the batch sizes of residual points $|B|$ and PDE terms $|I|$ under the constant memory cost ($|B| \times |I|$) can lead to reduced stochastic gradient variance and accelerated convergence compared to previous practices that use SGD over points only.
 
We assume that the total full batch is with $N_r$ residual points $\{\bx_n\}_{n=1}^{N_r}$ and $N_{\mathcal{L}}$ PDE terms $\mathcal{L} = \sum_{i=1}^{N_{\mathcal{L}}}\mathcal{L}_i$, then the loss function for PINN optimization is
$
\frac{1}{2N_rN^2_{\mathcal{L}}}\sum_{n=1}^{N_r}\left(\mathcal{L}u_\theta(\bx_n) - R(\bx_n)\right)^2,
$
where we normalize over both the number of residual points $N_r$ and the number of PDE terms $N_\mathcal{L}$, which does not impact the directions of the gradients.
More specifically, the original PDE is $\sum_{i=1}^{N_\mathcal{L}} \mathcal{L}_i u(\bx) = R(\bx)$, we normalize it into $\frac{1}{N_{\mathcal{L}}}\sum_{i=1}^{N_\mathcal{L}} \mathcal{L}_i u(\bx) = \frac{1}{N_{\mathcal{L}}}R(\bx)$. The reason for normalization lies in the increase of both dimensionality and the number of PDE terms. As these increase, the scale of the residual loss and its gradient becomes exceptionally large, leading to the issue of gradient explosions. Therefore, we normalize the loss by the dimensionality and the number of PDE terms to mitigate this problem. Additionally, during testing, we focus on relative errors, and normalization here is, in essence, a normalization concerning dimensionality following the relative error we care about.
The full batch gradient is given by
\begin{equation}\label{eq:full_batch_grad_theory}
\small
\begin{aligned}
g(\theta) &:= \frac{1}{N_rN^2_{\mathcal{L}}}\sum_{n=1}^{N_r}\left(\mathcal{L}u_\theta(\bx_n) - R(\bx_n)\right)\frac{\partial}{\partial \theta}\mathcal{L}u_\theta(\bx_n)= \frac{1}{N_rN^2_{\mathcal{L}}}\sum_{n=1}^{N_r}\left(\mathcal{L}u_\theta(\bx_n) - R(\bx_n)\right)\left(\sum_{i=1}^{N_{\mathcal{L}}}\frac{\partial}{\partial \theta}\mathcal{L}_iu_\theta(\bx_n)\right).
\end{aligned}
\end{equation}
If the random index sets $I \subset \{1,2,\cdots,N_{\mathcal{L}}\}$ and $B\subset \{1,2,\cdots,N_r\}$ are chosen, then the SDGD gradient with $|B|$ residual points and $|I|$ PDE terms using these index sets is
\begin{equation}\label{eq:algo1_grad_theory}
g_{B, I}(\theta) = \frac{1}{|B||I|N_{\mathcal{L}}}\sum_{n \in B}\left(\mathcal{L}u_\theta(\bx_n) - R(\bx_n)\right)\left(\sum_{i \in I}\frac{\partial}{\partial \theta}\mathcal{L}_iu_\theta(\bx_n)\right),
\end{equation}
where $|\cdot|$ computes the cardinality of a set. It is straightforward to show that $\mathbb{E}_{B,I}[g_{B,I}(\theta)] = g(\theta)$.
\begin{theorem}\label{thm:variance_1}
For the random index sets $(B, I)$ where $I\subset \{1,2,\cdots,N_{\mathcal{L}}\}$ is that for indices of PDE terms/dimensions and $B\subset \{1,2,\cdots,N_r\}$ is that for indices of residual points, then
\begin{equation}
\mathbb{V}_{B,I}[g_{B,I}(\theta)] = \frac{C_1|I| + C_2|B| + C_3}{|B||I|},
\end{equation}
where $\mathbb{V}$ computes the variance of a random variable, $C_1, C_2, C_3$ are constants independent of $B,I$ but dependent on model parameters $\theta$.
\end{theorem}
\begin{proof}
We prove the theorem in \ref{appendix:variance_1}.
\end{proof}
Theorem \ref{thm:variance_1} focuses on the gradient variance after fixing the current model parameter $\theta$, i.e., given $\theta$ at the current epoch, we investigate the SDGD gradient variance at the next optimization round.

Intuitively, the variance of the stochastic gradient tends to decrease as the batch sizes ($|B|, |I|$) increase, and it converges to zero for $|B|, |I| \rightarrow \infty$, where we considered sampling with replacement to make the theorem clear. Further, the SDGD gradient variance can be rewritten as $\frac{C_1}{|B|} + \frac{C_2}{|I|} + \frac{C_3}{|B||I|}$, where $B$ is the set for sampling residual points and $I$ is that for sampling PDE terms/dimensions. $C_1$ is the variance from points, and therefore, sampling more points, i.e., increasing $|B|$, can reduce it. Similarly, $C_2$ is the variance from PDE terms/dimensions. $C_3$ is related to the variance due to the correlation between sampling points and PDE terms/dimensions.

This verifies that SDGD can be regarded as another form of SGD over PDE terms, serving as a complement to the commonly used SGD over residual points. $|B|$-point + $|I|$-term with the same $|B| \times |I|$ quantity has the same memory cost. According to Theorem \ref{thm:variance_1}, under the same memory budget, i.e., $|B| \times |I|$ is fixed, then there exists a particular choice of batch sizes $|B|$ and $|I|$ that minimizes the gradient variance, in turn accelerating and stabilizing convergence. This is because the stochastic gradient variance $\mathbb{V}_{B,I}[g_{B,I}(\theta)]$ is a function of finite batch sizes $|B|$ and $|I|$, which therefore can achieve its minimum value at a certain choice of batch sizes.

Let us take the extreme cases as illustrative examples. The first extreme case is when the PDE terms have no variance, meaning that the terms $\frac{\partial}{\partial \theta}\mathcal{L}_iu_\theta(\bx_n)$ are identical for all $i$ after we fix $n$, i.e., $C_2 = 0$, and the SDGD variance becomes $\frac{C_1|I| + C_3}{|B||I|}$. In this case, if a memory budget of, for example, $|B||I| = 100$ units is given, the optimal choice would be to select 100 points and one PDE term with the minimum variance, i.e., we obtain the lowest SDGD variance by choosing $|B| = 100$ and $|I| = 1$. Choosing more PDE terms would not decrease the gradient variance and would be less effective than using the entire memory budget for sampling points. Conversely, if the points have no variance in terms of the gradient they induce, i.e., $C_1=0$, the optimal choice would be one point and 100 PDE terms, i.e., $|I| = 100$ and $|B| = 1$. In practice, the selection of residual points and PDE terms will inevitably introduce some variance. Therefore, in the case of a fixed memory cost, the choice of batch size for PDE terms and residual points involves a tradeoff. Increasing the number of PDE terms reduces the variance contributed by the PDE terms but also decreases the number of residual points, thereby increasing the variance of the points. Conversely, decreasing the number of PDE terms reduces the variance from the points but increases the variance from the PDE terms. Thus, there exists an optimal selection strategy to minimize the overall gradient variance since the choices of batch sizes are finite.

\subsection{Gradient Variance Bound and Convergence of SDGD}
To establish the convergence of unbiased stochastic gradient descent, we require either Lipschitz continuity of the gradients \cite{lei2019stochastic} or bounded variances \cite{fehrman2020convergence,mertikopoulos2020almost}, with the latter leading to faster convergence. To prove this property, we need to take the following steps. Firstly, we define the neural network serving as the surrogate model in PINN.
\begin{definition}\label{def:DNN}
(Neural Network). A deep neural network (DNN) $u_\theta:\bx=(\bx_{1},\dots,\bx_d)\in\mathbb{R}^d\mapsto u_\theta(\bx) \in \mathbb{R}$,
parameterized by $\theta$ of depth $L$ is the composition of $L$ linear functions with element-wise non-linearity $\sigma$, is expressed as $
    u_\theta(\bx)=W_L \sigma (W_{L-1} \sigma(\cdots \sigma(W_1\bx)\cdots ),
$
where $\bx\in\mathbb{R}^d$ is the input, and $W_l\in\mathbb{R}^{m_l \times m_{l-1}}$ is the weight matrix at $l$-th layer with $d=m_0$ and $m_L=1$. The parameter vector $\theta$ is the vectorization of the collection of all parameters. We denote $h$ as the maximal width of the neural network, i.e., $ h = \max(m_L, \cdots, m_0)$.
\end{definition}
For the nonlinear activation function $\sigma$, the residual ground truth $R(\bx)$, we assume the following:
\begin{assumption}\label{assumption:activation}
We assume that the activation function is smooth and $|\sigma^{(k)}(x)| \leq 1$ and $\sigma^{(k)}$ is 1-Lipschitz, for all $0 \leq k \leq n$, where $n$ is the highest order of the PDE under consideration, $\sigma^{(k)}$ denotes the $k$th order derivative of $\sigma$, e.g., the sine and cosine activations. We assume that $|R(\bx)| \leq R$ for all $\bx \in \Omega$ where $R$ is a constant.
\end{assumption}
\begin{remark}
The sine and cosine functions naturally satisfy the aforementioned conditions. As for the hyperbolic tangent (tanh) activation function, when the order $n$ of the PDE is determined, there exists a constant $C_n$ such that the activation function $C_n \tanh(x)$ satisfies the given assumptions. This constant can be absorbed into the weights of the neural network. This is because both the tanh function and its derivatives up to order $n$ are bounded.
\end{remark}
Assumption \ref{assumption:activation} gives the boundedness of the activation function $\sigma$ and the residual ground truth $R(\bx)$. Furthermore, to establish the stochastic gradient variance bounds for SDGD, we also need to assume the boundedness of the PDE differential operators $\mathcal{L}_i$ where $\mathcal{L} = \sum_{i=1}^{N_\mathcal{L}}\mathcal{L}_i$, as follows.
\begin{assumption}\label{assumption:operator} We assume the highest-order derivative in the PDE operator $\mathcal{L} = \sum_{i=1}^{N_\mathcal{L}} \mathcal{L}_i$ is $n$.
We view each $\mathcal{L}_i$ as $\left(\bx, u(\bx), \frac{\partial u(\bx)}{\partial \bx}, \cdots, \frac{\partial^n u(\bx)}{\partial \bx^n}\right) \mapsto \mathcal{L}_i\left(\bx, u(\bx), \frac{\partial u(\bx)}{\partial \bx}, \cdots, \frac{\partial^n u(\bx)}{\partial \bx^n}\right)$, which is a mapping $\mathbb{R}^{d + 1 + d + \cdots + d^n} \rightarrow \mathbb{R}$ since $\bx \in \mathbb{R}^d, \frac{\partial^n u(\bx)}{\partial \bx^n} \in \mathbb{R}^{d^n}$.
Then, we assume that each PDE operator $\mathcal{L}_i$ is bounded in the sense that if its input space $\left(\bx, u(\bx), \frac{\partial u(\bx)}{\partial \bx}, \cdots, \frac{\partial^n u(\bx)}{\partial \bx^n}\right) \in \mathbb{R}^{d + 1 + d + \cdots + d^n}$ is compact, then the output space of $\mathcal{L}_i$ is bounded.
\end{assumption}
The motivation is that $\bx$ is in a compact domain, i.e., the finite set of training residual/collocation points, since the training is finished in finite time. Due to Lemma \ref{lemma:nn_derivative_bound}, high-order derivatives of the neural network, i.e., $\frac{\partial^n u_\theta(\bx)}{\partial \bx^n}$, can also be bounded in a compact domain by bounded optimization trajectory assumption stated later in Theorem \ref{thm:convergence}, i.e., the parameter $\theta$ is assumed to be bounded during optimization. Combining these guarantees the PINN residual's and SDGD-PINN gradient's boundedness throughout the optimization, which can lead to convergence thanks to bounded gradient variance \cite{fehrman2020convergence,mertikopoulos2020almost}.
Assumption \ref{assumption:operator} is realistic since most PDE coefficients are constant or continuous functions. 

\begin{lemma}\label{lemma:nn_derivative_bound}
(Bounding the Neural Network Derivatives) Consider the neural network defined as Definition \ref{def:DNN} with parameters $\theta$, depth $L$, and width $h$, then the neural network's derivatives can be bounded as follows.
\begin{align}
\left\|\operatorname{vec}\left(\frac{\partial^n}{\partial\bx^n}u_\theta(\bx)\right)\right\| &\leq (n-1)!d^{n-1}(L-1)^{n-1}M(L) \prod_{l=1}^{L-1} M(l)^n,\\
\left\|\operatorname{vec}\left(\frac{\partial}{\partial \theta}\frac{\partial^n}{\partial\bx^n}u_\theta(\bx)\right)\right\|&\leq h^2n!d^{n}(L-1)^{n}M(L) \prod_{l=1}^{L-1} M(l)^{n+1} \max\left\{\Vert\bx\Vert, 1\right\},
\end{align}
where $M(l) = \max\left\{\Vert W_l \Vert, 1\right\}$, $n$ is the derivative order, the norms are all vector or matrix 2 norms and $\operatorname{vec}$ denotes vectorization of the high-order derivative tensor, and $h$ is the maximal width of the neural network, i.e., $ h = \max(m_L, \cdots, m_0)$.
\end{lemma}
\begin{proof}
The proof is presented in \ref{proof:nn_derivative_bound}.
\end{proof}
\begin{remark}
Intuitively, given the order of the PDE and the neural network structure, it is observed that both of them are finite. Consequently, the norms of the higher-order derivatives of the network are necessarily bounded by a quantity related to its weight matrices, the number of layers, the width, the PDE order, and the network input $\bx$. 
\end{remark}

From the boundedness of the derivatives of the neural network in PINN given by Lemma \ref{lemma:nn_derivative_bound}, coupled with our assumed boundedness of the PDE operator in Assumption \ref{assumption:operator}, we can infer that the PINN residual is also bounded, i.e., the boundedness of the numerical values obtained when the PDE operator acts on the neural network of PINN, which can further ensure bounded stochastic gradient variance during SDGD optimization.

Next, we define the SGD trajectory, and we will demonstrate the convergence of the SGD trajectory based on the stochastic gradients provided by our Algorithms \ref{algo:1} and \ref{algo:2}.

\begin{definition}
(SGD Trajectory) Given step sizes (learning rates) $\{\gamma_n\}_{n=1}^\infty$, and the initialized PINN parameters $\theta^1$, then the update rules of SGD using Algorithms \ref{algo:1} and \ref{algo:2} are
\begin{equation}\label{eq:SGD1}
\theta^{n+1} = \theta^{n} - \gamma^ng_{B,J}(\theta^{n}),
\end{equation}
\begin{equation}\label{eq:SGD2}
\theta^{n+1} = \theta^{n} - \gamma^ng_{B,I,J}(\theta^{n}),
\end{equation}
respectively. Here, the stochastic gradient produced by Algorithm \ref{algo:1} $g_{B, I}(\theta)$ is given in equation (\ref{eq:algo1_grad_theory})
, and the stochastic gradient produced by Algorithm \ref{algo:2} is
\begin{equation}\label{eq:algo2_grad_theory}
g_{B, I, J}(\theta) = \frac{1}{|B||I|N_{\mathcal{L}}}\sum_{n \in B}\left(\frac{N_{\mathcal{L}}}{|J|}\left(\sum_{j \in J}\mathcal{L}_ju_\theta(\bx_n)\right) - R(\bx_n)\right)\left(\sum_{i \in I}\frac{\partial}{\partial \theta}\mathcal{L}_iu_\theta(\bx_n)\right),
\end{equation}
where $B$ is the index set sampling over the residual points, while $I$ and $J$ sample the PDE terms in the backward and forward passes, respectively.
\end{definition}

Since the PINN optimization landscape is generally nonconvex, we consider SDGD's convergence to the regular minimizer defined as follows. Kawaguchi \cite{kawaguchi2016deep} demonstrates that such a ``regular minimizer" can be good enough in practice.
\begin{definition}\label{def:reg_min}
(Regular Minimizer) In PINN's general nonconvex optimization landscape. We say that a local minimizer $\theta^*$ is a regular minimizer if $H(\theta^*) \succ 0$, i.e., the PINN loss function's Hessian matrix at $\theta^*$ is positive definite.
\end{definition}

The following theorem demonstrates the convergence rate of SDGD in PINN training. The statement follows Theorem 4 in Mertikopoulos et al. \cite{mertikopoulos2020almost} and the proof is to check if the conditions in that theorem can be satisfied in our setting.

\begin{theorem}\label{thm:convergence}
(Convergence of SDGD under Bounded Gradient Variance) Assume Assumptions \ref{assumption:activation} and \ref{assumption:operator} hold, fix some tolerance level $\delta >0$, suppose that the SGD trajectories given in equations (\ref{eq:SGD1}, \ref{eq:SGD2}) are bounded, i.e., $\Vert W_l^n \Vert\leq M(l)$ for all epoch $n$ where the collection of all $\{W_l^n\}_{l=1}^L$ is $\theta^n$, and that the regular minimizer of the PINN loss is $\theta^*$ as defined in Definition \ref{def:reg_min}, and that the SGD step size follows the form $\gamma_n = \gamma / (n + m)^p$ for some $p \in (1/2,1]$ and large enough $\gamma, m > 0$, then
\begin{enumerate}
\item There exist neighborhoods $\mathcal{U}$ and $\mathcal{U}_1$ of $\theta^*$ such that, if $\theta^1 \in \mathcal{U}_1$, the event
\begin{equation}
E_\infty = \left\{\theta^n \in \mathcal{U} \ \text{for all } n \in \mathbb{N}\right\}
\end{equation}
occurs with probability at least $1 - \delta$, i.e., $\mathbb{P}(E_\infty | \theta^1 \in\mathcal{U}_1) \geq 1 - \delta$.
\item Conditioned on $E_\infty$, we have
\begin{equation}
\mathbb{E}[\Vert\theta^n - \theta^*\Vert^2|E_\infty] \leq \mathcal{O}(1/n^p).
\end{equation}
\end{enumerate}
\end{theorem}
\begin{proof}
The proof is presented in \ref{appendix:convergence}.
\end{proof}
\begin{remark}
The convergence rate of stochastic gradient descent is $\mathcal{O}(1/n^p)$, where $\mathcal{O}$ represents the constant involved, including the variance of the stochastic gradient. A larger variance of the stochastic gradient leads to slower convergence, while a smaller variance leads to faster convergence. The selection of the learning rate parameters $\gamma$ and $m$ depends on the specific optimization problem itself. For example, in a PINN problem, suitable values of $\gamma$ and $m$ ensure that the initial learning rate is around 1e-4, which is similar to the common practice for selecting learning rates.
\end{remark}
\begin{remark}
The theorem is mainly due to Theorem 4 in Mertikopoulos et al. \cite{mertikopoulos2020almost}, which is proved via checking if the required conditions are satisfied.
The convergence result is local and requires that the initialized neural network parameter $\theta^1$ in the first epoch be inside the neighborhood $\mathcal{U}_1$. A proper neural network weight/parameter initialization, e.g., Xavier initialization \cite{glorot2010understanding} used in our experiment, can mitigate this. This approach avoids the problem of initializing weights with too large or too small values, thus ensuring that the loss and gradient do not explode to infinity or vanish to zero during forward and backward passes, contributing to a small initial loss function value, which signifies a starting point close to a local minimum and stable gradient and fast convergence.
\end{remark}

\subsection{Discussion on Batch Sizes Selection}\label{sec:batch_size_selection}
We have introduced our SDGD algorithms with three batch sizes: $|B|$ for the batch size of residual points, $|I|$ for the dimensions included in the backward pass, and $|J|$ for the dimensions included in the forward pass; see the stochastic gradients in equations (\ref{eq:full_batch_grad_theory}, \ref{eq:algo1_grad_theory}, \ref{eq:algo2_grad_theory}). In this subsection, we further discuss how to choose batch sizes.

Firstly, both computational cost and memory consumption expand with the increase in batch sizes 
$|B|, |I|, |J|$. Additionally, SDGD's stochastic gradient variance diminishes with larger batch sizes, facilitating better convergence. Therefore, it represents a tradeoff between computational load and performance. 
Further, memory consumption is directly linked to $|B| \cdot |I|$, where $|B|$ represents the number of residual points and  $|I|$ involves the number of dimensions participating in the backward pass. $|J|$ has nothing to do with the memory cost since it is for the forward pass detached from the GPU after computation.

Users should determine the appropriate batch size based on their GPU memory capacity. Given a specified memory constraint, i.e., once the $|B| \cdot |I|$ value is determined, we recommend a balanced selection to ensure similarity between the two batch sizes, i.e., $|I| \approx |B|$. Extremely unbalanced cases, such as setting one batch size to a minimum of 1, are not advised. We explain the negative effect of unbalanced batch sizes as follows.
If the batch size for residual points, $|B|$, is too small, it can hinder the GPU from leveraging its advantage in large-batch parallel computing. The parallel nature of neural network inferences for each residual point cannot be used, resulting in slower execution due to the exceptionally small batch size of points $|B|$. 
Meanwhile, if one reduces the batch for sampling dimension $|I|$ too much, we cannot reuse the computation graphs for the derivatives along different dimensions, leading to a significant slowdown. Specifically, taking the high-dimensional Laplacian operator as an example, the neural network's second-order derivative with respect to the $j$-th dimension is given by
\begin{equation}
\begin{aligned}
\frac{\partial^2 u_{\theta}(\bx)}{\partial \bx_j^2} &=\sum_{l=1}^{L-1} 
(W_L\Phi_{L-1}(\bx)\cdots W_{l+1})
\text{diag}(\Psi_l(\bx)\cdots\Psi_1(\bx)(W_1)_{:,j})
(W_l\cdots \Phi_1(\bx)W_1),
\end{aligned}
\end{equation}
where
$\Phi_l(\bx) = \text{diag}[\sigma'(W_l\sigma(W_{l-1}\sigma(\cdots\sigma(W_1\bx))))] \in \mathbb{R}^{m_l\times m_l}$, and
$\Psi_l(\bx) = \text{diag}[\sigma''(W_l\sigma(W_{l-1}\sigma(\cdots\sigma(W_1\bx))))] \in \mathbb{R}^{m_l\times m_l}$.
For all dimensions $j$, the computations of $\Phi_l(\bx)$ and $\Psi_l(\bx)$ are shared. Thus, computing multiple derivatives on one computational graph is faster than computing them from scratch without reusing the shared part. Hence, we recommend sampling an appropriate dimension batch size, such as 100 or more, and we avoid choosing a too small $|I|$ to reuse the computation graph properly.

As for the choice of $|J|$ in Algorithm \ref{algo:2}, we recommend setting $|J| \approx |I|$. This is because both $|I|$ and $|J|$ represent the dimension batch size, ensuring that the computational load for both forward and backward passes is roughly equal and the variances of the two samplings are similar. Although $|J|$ does not affect memory consumption, setting it excessively large is not advisable, as it would resemble Algorithm \ref{algo:1}, being memory-efficient but suffering from slow computation due to excessive forward pass workload.

\subsection{Comparison with Latest Work}
Recently, a line of work has been proposed for high-dimensional PINN, including Random Smoothing PINN (RS-PINN) \cite{hu2023bias}, Hutchinson trace estimation (HTE) PINN \cite{hu2023hutchinson}, and Score-PINN \cite{hu2024score}. We compare these three latest papers with our SDGD proposed in this paper.
Generally speaking, we demonstrate that SDGD is still more general, universal, applicable, easily implementable, contributory to high-order and high-dimensional PINN's theory, and tends to have lower variance in most scenarios despite these latest follow-ups.

\subsubsection{Comparison with RS-PINN}
\textbf{RS-PINN \cite{hu2023bias}} is another way to speed up and scale up PINN for high-dimensional PDEs, first proposed by He et al. \cite{he2023learning}. However, we show that RS-PINN \cite{hu2023bias} is worse than SDGD-PINN in several aspects.
First, RS-PINN cannot deal with PDEs with more than second-order terms since the corresponding variance reduction form cannot be derived, while SDGD applies to arbitrary high-dimensional high-order PDEs.
Second, RS-PINN smoothes the network, and thus, it can only model smooth PDE solutions, i.e., RS-PINN is not a universal approximator. On the other hand, SDGD employs a regular neural network and is universal.
Third, RS-PINN requires sampling, which leads to bias when dealing with nonlinear PDEs like the Sine-Gordon equation, as pointed out in Section 4.3 of RS-PINN paper \cite{hu2023bias}. Fourth, we compare RS-PINN with SDGD-PINN in Table \ref{tab:adv} in this paper, where He et al. \cite{he2023learning} is the RS-PINN. However, we demonstrate that SDGD is much faster than RS-PINN.

Overall, the bias, over-smoothness, and limitation on lower-than-second-order derivatives threaten RS-PINN's application. SDGD is better in these senses. SDGD also outperforms RS-PINN significantly in numerical experiments.

\subsubsection{Comparison with HTE}
\textbf{HTE's \cite{hu2023hutchinson}} is another method for accelerating high-dimensional PINN.
We compare SDGD with HTE in (1) gradient variance; (2) applicability and generality; (3) flexibility in code implementation; (4) numerical experiment; (5) theoretical contribution. Overall, SDGD is better than HTE \cite{hu2023hutchinson} in these crucial aspects.

\begin{enumerate}
\item Gradient variance: We first compare the variance of SDGD and HTE theoretically. Overall, SDGD and HTE have distinct pros and cons, but SDGD will be more stable in most high-dimensional PDEs than HTE.

Hu et al. \cite{hu2023hutchinson} gave examples where SDGD/HTE can be better and cases where they perform the same in their Section 3.3.2, which is the comprehensive and fair comparison showing that HTE is not necessarily better than SDGD.

We further demonstrate that SDGD's variance can be more stable than HTE in most extremely high-dimensional problems. 
Using the setting of second-order PDEs in Section 3.3.2 of HTE \cite{hu2023hutchinson}, SDGD and HTE both estimate the Hessian trace stochastically. 
The HTE paper shows that HTE's variance comes from the off-diagonal terms, whose number grows quadratically with dimensionality.
SDGD's variance comes from the variability of on-diagonal terms, whose number grows linearly with dimension. 
Hence, with growing PDE dimensionality, HTE tends to suffer from more variance due to a faster-growing number of off-diagonal terms. In other words, the number of off-diagonal terms injecting noise into HTE's stochastic gradient grows quadratically fast, while the number of on-diagonal terms injecting noise into SDGD only grows linearly.

To conclude, regarding gradient variance, SDGD and HTE have distinct pros and cons, but SDGD will be more stable in most high-dimensional PDEs than HTE.

\item Applicability and generality: Next, we compare the applicability of SDGD and HTE. SDGD can be applied to arbitrary high-dimensional and high-order PDEs thanks to the general decomposition of PDE operators $\mathcal{L} = \sum_{i=1}^{N_\mathcal{L}} \mathcal{L}_i$ based on dimensionality. In contrast, HTE can only be used for first-order, second-order, and biharmonic PDEs, which enables the Tensor-Vector Product proposed in HTE. 

More specifically, HTE \cite{hu2023hutchinson}'s Section 3.5.2 acknowledges HTE's limitation. In contrast, we are always emphasizing SDGD's generality to arbitrary PDEs throughout this paper.

\item Implementation flexibility: Furthermore, SDGD is more flexible and can be implemented in both PyTorch \cite{paszke2019pytorch} and JAX \cite{jax2018github}. However, HTE \cite{hu2023hutchinson} can only be implemented with JAX \cite{jax2018github} Taylor mode auto-differentiation. The HTE paper \cite{hu2023hutchinson} acknowledges HTE's limitation in Section 3.2.3.

In fact, PyTorch \cite{paszke2019pytorch} is still more widely used than JAX \cite{jax2018github} and has minimal requirements on the computer system. Hence, SDGD is more implementable than HTE.

Additionally, regarding implementation, SDGD and HTE are compared in JAX implementation using Section 5.1's setting in the HTE paper \cite{hu2023hutchinson}. We implemented other experiments (Sections 5.2, 5.3, and 5.4) in PyTorch since our baselines are also implemented in PyTorch. Thus, we keep the deep learning package the same for fair comparison and reproducible research.

\item Numerical experiments: Meanwhile, we compare HTE and SDGD experimentally. As shown by Table 1 in the HTE paper \cite{hu2023hutchinson}, if both are implemented in JAX \cite{jax2018github} and under fair comparison with the same memory and speed budgets, they perform similarly. So, SDGD and HTE are comparable in this case. 
 
But, as mentioned before, SDGD is more applicable and implementable than HTE, and SDGD generally has a low variance with the growing PDE dimension.

\item Theoretical contribution: SDGD's Section 4 provides a convergence theory for general high-order and high-dimensional PINN. This theoretical contribution can also be applied to bound the HTE gradient variance. It can also guarantee future work on general high-order and high-dimensional PINN.
\end{enumerate}

Overall, SDGD is better than HTE, especially regarding the variance in growing PDE dimension, generality, and implementation.

\subsubsection{Comparison with Score-PINN}
\textbf{Score-PINN \cite{hu2024score}} is an application that works on high-dimensional Fokker-Planck (FP) equations by adopting the speed-up and scale-up techniques we have developed, i.e., SDGD can be applied to it. It is not proposing any general methodology of high-dimensional PDE solver, but it's novelty is to adapt existing techniques to FP equation to solve the numerical issue due to extremely small solution value.

In sum, Score-PINN \cite{hu2024score} is an application paper that uses SDGD or HTE for a specific class of FP equations. It is not a methodology paper like SDGD. Score-PINN is also based on SDGD.

\section{Experiments}
In this section, we conduct extensive experiments to demonstrate the stable and fast convergence of SDGD over several nonlinear PDEs. In particular, 
we compare our PINN-based method with other non-PINN mainstream approaches for solving high-dimensional PDEs \cite{beck2021deep, han2018solving, raissi2018forward}, especially in terms of accuracy. We also demonstrate SDGD's convergence in 100,000-dimensional nonlinear PDEs with complicated exact solutions.
All the experiments in the main text are done on an NVIDIA A100 GPU with 80GB memory using PyTorch implementation \cite{paszke2019pytorch}.

\subsection{SDGD Can Deal with Inseparable and Effectively High-Dimensional PDE Solutions}

Herein, we test the ability of SDGD  to deal with linear/nonlinear PDEs with nonlinear, inseparable, and effectively high-dimensional PDE exact solutions in PINN's residual loss. Specifically, we consider the following exact solution:
\begin{equation}
u_{\text{exact}}(\bx) = \left(1 - \Vert \bx \Vert_2^2\right)\left(\sum_{i=1}^{d-1}  c_i \sin(\bx_i +\cos(\bx_{i+1})+\bx_{i+1}\cos(\bx_i))\right),
\end{equation}
where $c_i \sim \mathcal{N}(0, 1)$. Here the term $\left(1 - \Vert \bx \Vert_2^2\right)$ is for a zero boundary condition on the unit sphere and for preventing the boundary from leaking the information of the solution, as we want to test SDGD's ability to fit the residual where sampling over dimensions is employed. Furthermore, the following PDEs are considered in the unit ball $\mathbb{B}^d$, all of which are associated with a zero boundary condition on the unit sphere:
\begin{itemize}
\item Poisson equation
\begin{equation}
\Delta u(\bx) = g(\bx), \quad \bx\in \mathbb{B}^d,
\end{equation}
where $g(\bx) = \Delta u_{\text{exact}}(\bx)$.
\item Allen-Cahn equation
\begin{equation}
\Delta u(\bx) + u(\bx) - u(\bx)^3 = g(\bx), \quad \bx\in \mathbb{B}^d,
\end{equation}
where $g(\bx) = \Delta u_{\text{exact}}(\bx) + u_{\text{exact}}(\bx) - u_{\text{exact}}(\bx)^3$.
\item Sine-Gordon equation
\begin{equation}
\Delta u(\bx) + \sin\left(u(\bx) \right) = g(\bx), \quad \bx \in \mathbb{B}^d,
\end{equation}
where $g(\bx) = \Delta u_{\text{exact}}(\bx) + \sin\left(u_{\text{exact}}(\bx) \right)$.
\end{itemize}
We would like to emphasize that the PDEs and the exact PDE solutions used here are highly nontrivial and challenging.
\begin{itemize}
\item Nonlinear, inseparable, and effectively high-dimensional PDE exact solution: 
This PDE cannot be reduced to a lower-dimensional problem, and the exact solution of the PDE cannot be decomposed into lower-dimensional functions. It is worth noting that in the exact solution, all pairs of variables, both $\bx_i$ and $\bx_{i+1}$, exhibit pairwise interactions and are highly coupled.
\item Anisotropic and nontrivial PDE exact solution: In addition, due to the random coefficients $c_i \sim \mathcal{N}(0, 1)$ along different dimensions, the exact solution is anisotropic, which poses an additional challenge to SDGD. The exact solution is also highly nontrivial; its value's standard deviation is large in high-dimensional spaces.
\item PDEs with different levels of nonlinearities: We test both linear and nonlinear PDEs to verify SDGD's robustness.
\item Zero boundary condition to test SDGD on the residual part: In a $d$-dimensional problem, the boundary is typically $(d-1)$-dimensional. Therefore, if the boundary conditions are nontrivial, the boundary loss will leak information at a rate of $(d-1)/d$. To assess the impact of sampling the dimension in the residual loss in SDGD, we maintain a zero boundary condition to eliminate its interference. This also makes the PDE more challenging, as we have no knowledge of the exact solution from the boundary.
\item Unable to solve via traditional methods: These elliptic PDEs are traditionally solved by finite difference or finite element methods, which suffer from the curse of dimensionality, whose costs grow exponentially with respect to the dimension. The related work mentioned before, e.g., \cite{beck2021deep,han2018solving,raissi2019physics}, is mainly tailored for time-dependent parabolic PDEs. Thus, traditional methods directly fail in these cases.
\end{itemize}

The training details are as follows. The model is a 4-layer multi-layer perceptron network with 128 hidden units, which is trained via Adam \cite{kingma2014adam} for 10K epochs, with an initial learning rate 1e-3, which linearly decays to zero at the end of the optimization. We select 100 random residual points at each Adam epoch (i.e., $|B| = 100$) and 20K fixed testing points uniformly from the unit ball. We utilize Algorithm \ref{algo:3} with a minibatch of 100 dimensions to randomly sample the second-order derivatives along various dimensions in the Laplacian operator in all the PDEs, i.e., $|I| = 100$. For vanilla PINN baseline \cite{raissi2019physics}, we use full dimension gradient descent with $|I|$ equals the PDE dimensionality. We adopt the following model structure to automatically satisfy the zero boundary condition, which helps the model avoid boundary loss and avoid sampling the boundary points \cite{lu2021physics}:
\begin{equation}
u^{\text{SDGD}}_\theta(\bx) = (1 - \Vert\bx\Vert_2^2) u_\theta(\bx),
\end{equation}
where $u_\theta(\bx)$ is the neural network and $u^{\text{SDGD}}_\theta(\bx)$ is the boundary-augmented model. The hard constraint technique for PINN is popular to avoid the additional boundary loss \cite{lu2021physics}. We repeat our experiment 5 times with 5 independent random seeds.

\begin{table}[htbp]
\centering
\begin{tabular}{|c|c|c|c|c|c|c|c|}
\hline
\multirow{6}{*}{Vanilla PINN} & PDE & Metric & 100 D & 1,000 D & 5,000 D & 10,000 D & 100,000 D \\ \cline{2-8} 
 & Poisson & Rel. $L_2$ Error & 7.189E-3 & 5.609E-4 & 1.768E-3 & N.A. & N.A. \\ \cline{2-8} 
 & Allen-Cahn & Rel. $L_2$ Error & 7.187E-3 & 5.617E-4 & 1.773E-3 & N.A. & N.A.  \\ \cline{2-8} 
 & Sine-Gordon & Rel. $L_2$ Error & 7.192E-3 & 5.642E-4 & 1.782E-3 & N.A. & N.A. \\ \cline{2-8} 
 &  & Time (Hour) & 0.05 & 4.75 & 30.54 & N.A. & N.A. \\ \cline{2-8} 
 &  & Memory (MB) & 1328 & 4425 & 56563 & $>$81252 & $>$81252 \\ \hline
\hline
\multirow{6}{*}{SDGD (Ours)} & PDE & Metric & 100 D & 1,000 D & 5,000 D & 10,000 D & 100,000 D \\ \cline{2-8} 
 & Poisson & Rel. $L_2$ Error & 7.189E-3 & 5.611E-4 & 1.758E-3 & 1.850E-3 & 2.175E-3 \\ \cline{2-8} 
 & Allen-Cahn & Rel. $L_2$ Error & 7.187E-3 & 5.615E-4 & 1.762E-3 & 1.864E-3 & 2.178E-3  \\ \cline{2-8} 
 & Sine-Gordon & Rel. $L_2$ Error & 7.192E-3 & 5.641E-4 & 1.795E-3 & 1.854E-3 & 2.177E-3 \\ \cline{2-8} 
 &  & Time (Hour) & 0.05 & 0.75 & 1.18 & 1.5 & 12 \\ \cline{2-8} 
 &  & Memory (MB) & 1328 & 1788 & 3335 & 4527 & 32777 \\ \hline
\end{tabular}
\caption{Relative $L_2$ error, time, and memory costs for vanilla PINNs \cite{raissi2019physics} and SDGD (Ours) across different PDEs under various high dimensions.
Vanilla PINNs \cite{raissi2019physics} run out-of-memory for an A100 GPU with 81252MB memory in more than 5,000 D, while SDGD scales up to 100,000 D.
SDGD's errors are similar to those of vanilla PINNs in relatively lower dimensions (100 D, 1,000 D, and 5,000 D) while being much faster and more memory-efficient.}
\label{tab:1}
\end{table}


\begin{figure}[htbp]
\centering
\includegraphics[scale=0.57]{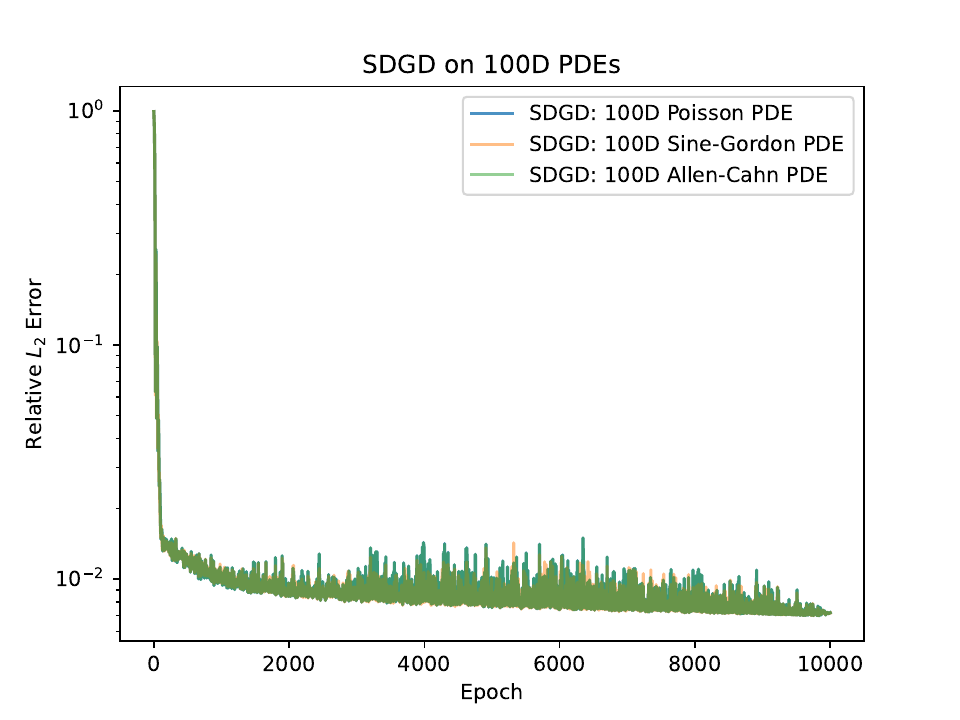}
\centering
\includegraphics[scale=0.57]{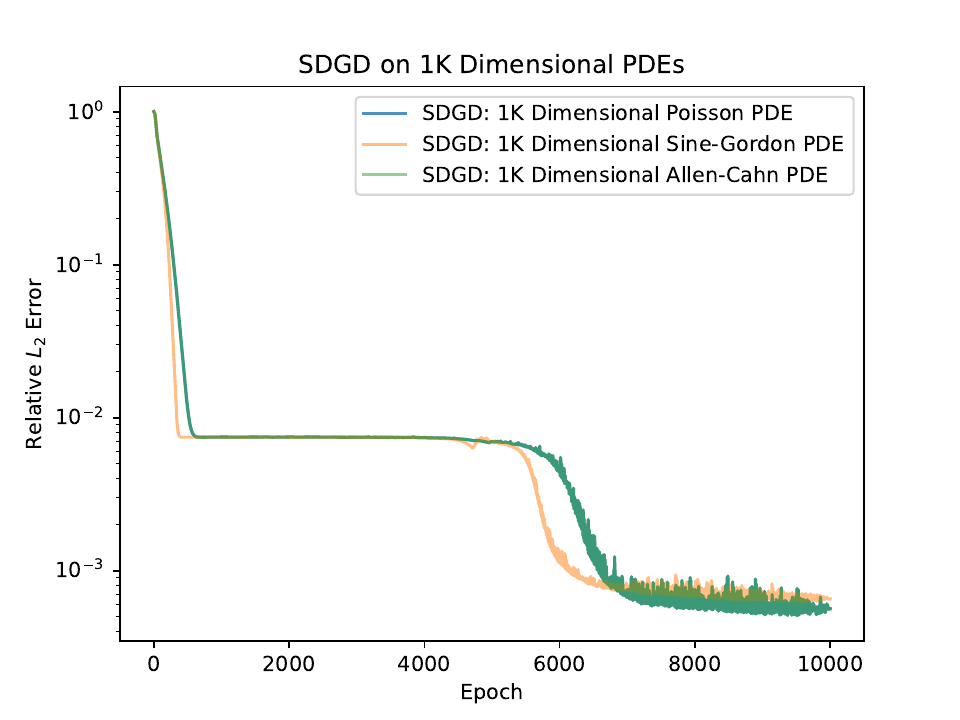}
\centering
\includegraphics[scale=0.57]{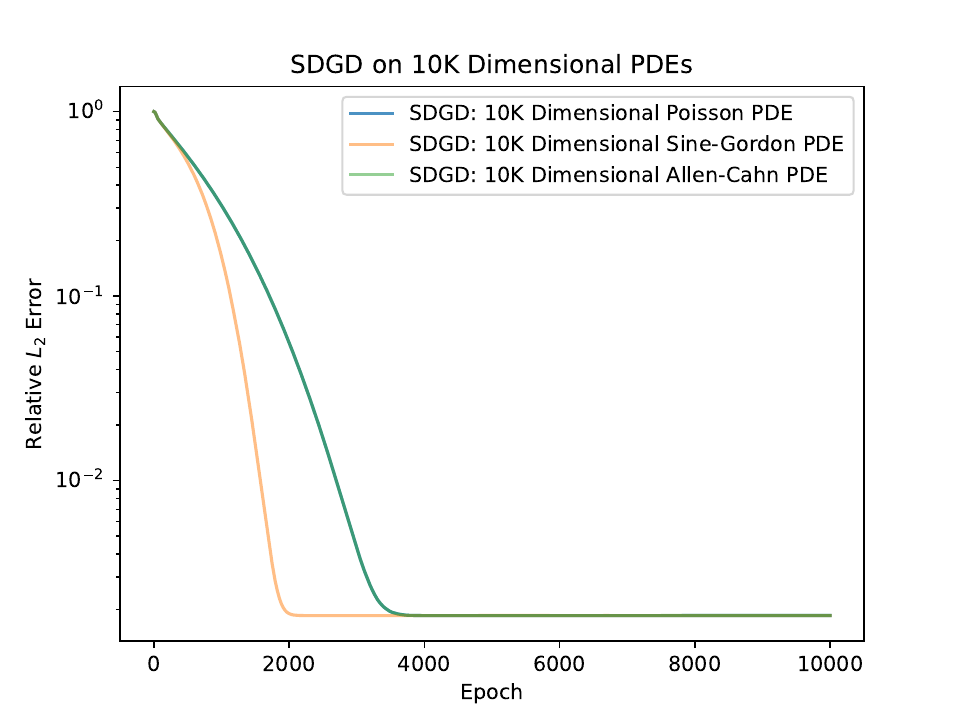}
\centering
\includegraphics[scale=0.57]{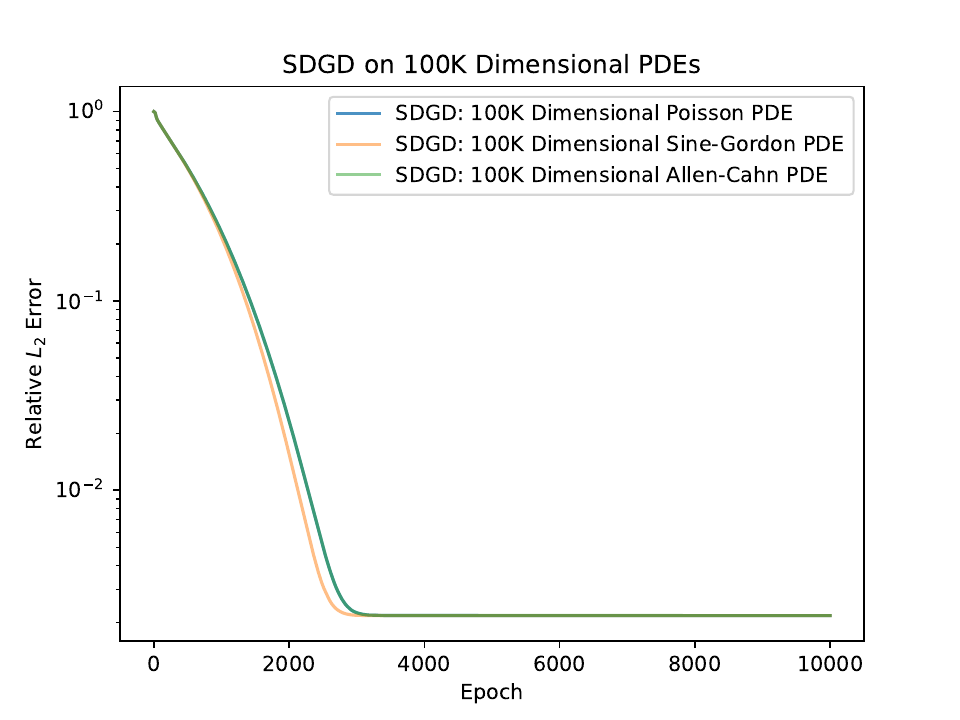}
\caption{SDGD convergence curve for the Poisson, Allen-Cahn, and Sine-Gordon PDEs in 100, 1K, 10K, and 100K dimensions.}
\label{fig:1}
\end{figure}

The test relative $L_2$ errors and the convergence time are shown in Table \ref{tab:1}, while the convergence curve of SDGD is shown in Figure \ref{fig:1}. The main observations from the results are as follows.
\begin{itemize}
\item Stable performances: 
SDGD demonstrates remarkable stability across various dimensions and different nonlinear PDEs, with relative $L_2$ errors testing at approximately 1e-3. 
Furthermore, given that we have provided the same exact solution for each PDE and have kept the training and testing point samples fixed by using the same random numbers, the convergence curves of SDGD for different PDEs are remarkably similar. This is especially noticeable for Poisson and Allen-Cahn, where the convergence curves exhibit minimal distinctions. This highlights the stability of SDGD in the face of PDE nonlinearity.
\item Linear computational cost: Furthermore, SDGD's memory and optimization time costs exhibit linear growth with respect to dimensionality. For instance, when transitioning from 10K dimensions to 100K dimensions, both memory and time costs increase by less than a factor of 10.
\item SDGD scales up PINNs to extremely high dimensions: In the case of 100K dimensions, we opted for $|I|=100$ dimensions for SDGD. If we were to employ the conventional full-batch gradient descent across all dimensions, even with just one residual point, the memory cost would surpass SDGD's current setting by a factor of 10. This greatly exceeds the 80GB memory limit of an A100 GPU. Thus, SDGD allows for the scalability of PINN to handle large-scale PDE problems.
\item Smoother convergence curve in higher dimensions: We observe that the convergence curves of SDGD are smoother in high-dimensional scenarios, which could be attributed to the tendency of high-dimensional functions to be smoother thanks to the blessing of dimensionality.
\item Comparison with vanilla PINNs \cite{raissi2019physics}: With regards to vanilla PINNs, we note a considerable decrease in speed and a swift escalation in memory usage with the growth of dimensionality. This occurs due to the necessity of calculating the entire second-order Laplacian operator in regular PINNs. This leads to a quadratic rise in computational expense with dimension due to the growing network size at the input layer and the number of second-order derivatives in the high-dimensional Laplacian operator. Consequently, when dealing with 10,000 dimensions, regular PINNs surpass the memory capacity of an A100 GPU of 81252 MB. Moreover, in scenarios where vanilla PINNs can handle dimensions less than 5,000 D, the performance of SDGD is comparable, with the added benefits of significant acceleration and reduced memory consumption. In 100 D, since our batch size for dimension is $|I|=100$, SDGD and vanilla PINN are the same algorithm. These results demonstrate that SDGD scales up and speeds up PINNs in very high dimensions.
\end{itemize}
In sum, SDGD with PINN can tackle the CoD in solving nonlinear high-dimensional PDEs with inseparable and complicated solutions.

\subsubsection{Ablation Study 1: Effect of the Batch Sizes for Residual Points and Sampled Dimensions}

\begin{table}[htbp]
\centering
\footnotesize
\begin{tabular}{|c|ccccc|ccc|}
\hline
Remark & \multicolumn{5}{c|}{$|I| \cdot |B| = 10^4$} & \multicolumn{3}{c|}{Lower Batch Sizes} \\ \hline
\# of Sampled Dimensions $|I|$ & \multicolumn{1}{c|}{100} & \multicolumn{1}{c|}{1,000} & \multicolumn{1}{c|}{10} & \multicolumn{1}{c|}{10,000} & 1 & \multicolumn{1}{c|}{100} & \multicolumn{1}{c|}{10} & 10 \\ \hline
\# of Residual Points $|B|$ & \multicolumn{1}{c|}{100} & \multicolumn{1}{c|}{10} & \multicolumn{1}{c|}{1,000} & \multicolumn{1}{c|}{1} & 10,000 & \multicolumn{1}{c|}{10} & \multicolumn{1}{c|}{100} & 10 \\ \hline
Relative $L_2$ Error & \multicolumn{1}{c|}{2.177E-3} & \multicolumn{1}{c|}{2.190E-3} & \multicolumn{1}{c|}{2.185E-3} & \multicolumn{1}{c|}{2.180E-3} & N.A. & \multicolumn{1}{c|}{2.203E-3} & \multicolumn{1}{c|}{2.189E-3} & 2.794E-3 \\ \hline
Speed (second per iteration) & \multicolumn{1}{c|}{4.35s/it} & \multicolumn{1}{c|}{3.21s/it} & \multicolumn{1}{c|}{56.13s/it} & \multicolumn{1}{c|}{31.25s/it} & 546.56s/it & \multicolumn{1}{c|}{0.85s/it} & \multicolumn{1}{c|}{5.72s/it} & 0.45s/it \\ \hline
Time (Hour) & \multicolumn{1}{c|}{12.08} & \multicolumn{1}{c|}{8.92} & \multicolumn{1}{c|}{155.92} & \multicolumn{1}{c|}{86.81} & 1518.22 & \multicolumn{1}{c|}{2.37} & \multicolumn{1}{c|}{15.89} & 1.25 \\ \hline
Memory (MB) & \multicolumn{1}{c|}{32777} & \multicolumn{1}{c|}{32481} & \multicolumn{1}{c|}{33056} & \multicolumn{1}{c|}{32163} & 33458 & \multicolumn{1}{c|}{5823} & \multicolumn{1}{c|}{5965} & 1849 \\ \hline
\end{tabular}
\caption{SDGD results for the 100,000-dimensional Sine-Gordon equation with anisotropic solutions under various batch sizes.}
\label{tab:ablation1}
\end{table}

In this section, we investigate the impact of the batch sizes of residual points $|B|$ and dimensions sampled $|I|$ in SDGD on its convergence results. We adopt the same experimental settings of 100,000-dimensional Sine-Gordon equations and hyperparameter choices as in Section 5.1, except we control these two batch sizes $|B|$ and $|I|$ in Algorithm \ref{algo:3}. The results are presented in Table \ref{tab:ablation1}, where we manipulate the batch sizes in eight different cases with $(|I|, |B|)$ = (100, 100) or (1000, 10) or (10, 1000) or (10000, 1) or (1, 10000) or (100, 10) or (10, 100) or (10, 10). We remark that the total batch size $|I| \cdot |B| = 10^4$ for the first five cases. So, they should have similar memory costs according to Section \ref{sec:batch_size_selection}. For the other three cases with $(|I|, |B|)$ = (100, 10) or (10, 100) or (10, 10), we remark that we chose even smaller total batch sizes. For each case, we report the test relative $L_2$ error, speed (second per iteration), total running time for 10K iterations, and GPU memory cost in MB.
\begin{itemize}
\item Firstly, regarding performance, the convergence results of the final relative $L_2$ error are quite consistent among various settings, indicating that SDGD is robust to the choice of batch size. Particularly, smaller batch sizes like $(|I|=100, |B|=10)$, $(|I|=10, |B|=100)$, and $(|I|=10, |B|=10)$ converge only slightly worse than larger ones but with faster speeds, suggesting that selecting fewer points and dimensions is sufficient when dealing with high-dimensional PINNs.
\item As for memory consumption, let us consider the first five settings; if the total batch size (the numerical value of $|I| \cdot |B| = 10^4$ for these five cases) remains the same, the memory usage is generally comparable. When selecting more points, additional memory is consumed due to the inherent memory usage of each point, but this constitutes only a small fraction of the total memory.
\item Regarding speed, we observe that reducing any batch sizes, $|I|$ or $|B|$, accelerates the training. Specifically, the case ``$(|I|, |B|)$ = (100, 100)" runs slower than the case ``$(|I|, |B|)$ = (100, 10)" and the case ``$(|I|, |B|)$ = (10, 100)", and all of them run slower than the case with smallest batch size ``$(|I|, |B|)$ = (10, 10)".
\item Consider the five cases with the same total batch size $|I| \cdot |B| = 10^4$. We observed that extremely unbalanced batch sizes, e.g., ``$(|I|, |B|)$ = (10000, 1)" and ``$(|I|, |B|)$ = (1, 10000)", result in particularly slow speeds. Conversely, in cases where a more balanced batch size is chosen, e.g., ``$(|I|, |B|)$ = (100, 100)", the computation tends to run faster.
If the batch size for residual points $|B|$ is excessively small, the advantage of GPU parallel computing with large batches of residual points cannot be utilized, leading to a slowdown in speed. If the batch size for dimensions $|I|$ is too small, it sacrifices the reuse of the computational graph for calculating the second-order derivatives along each dimension, resulting in slower execution, consistent with our analysis in Section \ref{sec:batch_size_selection}.
\item Furthermore, in similarly extreme scenarios, selecting an exceptionally small batch size for the sampled dimension results in even slower speeds, i.e., the case ``$(|I|, |B|)$ = (1, 10000)" is much slower. In this example, the slowdown due to not reusing the computation graph for each dimension's second-order derivatives is much more obvious. 
Decreasing the batch size too much for either residual points or dimensions is not ideal, and which one has a greater negative impact depends on the specific setting and model architecture. Also, the case ``$(|I|, |B|)$ = (10, 1000)" is much slower than the case `$(|I|, |B|)$ = (1000, 10)" points to the same fact.
\end{itemize}
In summary, this section demonstrates the stability of SDGD concerning batch size. It highlights that achieving good convergence results in high-dimensional PINNs can be accomplished with a small batch of randomly selected points and dimensions using SDGD. We also validate the guideline for batch size selection in Section \ref{sec:batch_size_selection}.

\subsubsection{Ablation Study 2: Comparisons between Algorithms \ref{algo:1}, \ref{algo:2}, and \ref{algo:3}}\label{sec:exp_bias_speed_tradeoff}

\begin{table}[htbp]
\centering
\begin{tabular}{|c|c|c|c|c|}
\hline
Algorithm & Metric & 1,000 D & 10,000 D & 100,000 D \\ \hline
\multirow{3}{*}{Algorithm \ref{algo:1}} & Relative $L_2$ Error & 5.134E-4 & N.A. & N.A. \\ \cline{2-5} 
 & Time (Hour) & 4.36 & $>$24 & $>$120 \\ \cline{2-5} 
 & Memory (MB) & 1788 & 4527 & 32777 \\ \hline\hline
\multirow{3}{*}{Algorithm \ref{algo:2}} & Relative $L_2$ Error & 5.193E-4 & 1.682E-3 & 2.039E-3 \\ \cline{2-5} 
& Time (Hour) & 1.08 & 1.8 & 14.4 \\ \cline{2-5} 
& Memory (MB) & 1788 & 4527 & 32777 \\ \hline\hline
\multirow{3}{*}{Algorithm \ref{algo:3}} & Relative $L_2$ Error & 5.641E-4 & 1.854E-3 & 2.177E-3 \\ \cline{2-5} 
 & Time (Hour) & 0.75 & 1.5 & 12 \\ \cline{2-5} 
 & Memory (MB) & 1788 & 4527 & 32777 \\ \hline
\end{tabular}
\caption{Comparisons between Algorithms \ref{algo:1}, \ref{algo:2}, and \ref{algo:3} on the Sine-Gordon equations in various dimensions.}
\label{tab:ablation2}
\end{table}

We have demonstrated the rapid convergence of SDGD on complex nonlinear PDEs using Algorithm \ref{algo:3} and its stability in particularly high dimensions across various nonlinear PDEs. Next, we will compare the speed and efficacy of Algorithms \ref{algo:1}, \ref{algo:2}, and \ref{algo:3}. Given the stability of SDGD across various PDEs, we will employ the example of the nonlinear Sine-Gordon PDE in various dimensions, as discussed in Section 5.1. We have maintained identical experimental hyperparameters to focus specifically on comparing the algorithms as in the previous section. Furthermore, for the batch sizes of the three algorithms, including the batch sizes for residual/collocation points ($|B|$), PDE terms/dimensions in the forward pass ($|J|$), and PDE terms/dimensions for the backward pass ($|I|$), we consistently chose $|B| = |I| = |J| = 100$, ensuring a fair comparison.

The results are shown in Table \ref{tab:ablation2}. The memory cost of the three algorithms is the same because their values of $|I|$ and $|B|$ are equivalent, representing the number of residual points and dimensions involved in the backward pass, respectively. Regarding speed, Algorithm \ref{algo:1} necessitates a forward pass for all dimensions. Despite being memory-efficient, its speed is compromised by the massive forward pass computation, particularly in extremely high dimensions. Algorithm \ref{algo:3} is slightly faster than Algorithm \ref{algo:2}. This is attributed to the fact that the former samples only once, utilizing the same set of dimensions for both forward and backward passes. In contrast, the latter requires two separate sets of dimensions for these passes. Regarding performance, Algorithm \ref{algo:1} slightly outperforms Algorithm \ref{algo:2}, which, in turn, is slightly better than Algorithm \ref{algo:3}. This is because we employ more randomness to accelerate at the expense of introducing a certain degree of gradient variance. Nevertheless, Algorithm \ref{algo:3}, the fastest in terms of speed, proves to be quite effective, with marginal differences compared to more precise algorithms.

\begin{table}[htbp]
\centering
\begin{tabular}{|c|c|c|c|c|c|}
\hline
Method & PDE  & Dim=$10^1$       & Dim=$10^2$   & Dim=$10^3$   & Dim=$10^4$   \\ \hline
FBSNN \cite{raissi2018forward} & Allen-Cahn & 1.339E-3	& 2.184E-3	&7.029E-3&	5.392E-2
 \\ \hline
DeepBSDE \cite{han2018solving} & Allen-Cahn & 4.581E-3	&2.516E-3&	2.980E-3	&2.975E-3
 \\ \hline
DeepSplitting \cite{beck2021deep} & Allen-Cahn & 3.642E-3	&1.563E-3	&2.369E-3&	2.962E-3
 \\ \hline
SDGD (Ours) & Allen-Cahn & \textbf{7.815E-4}&	\textbf{3.142E-4}&\textbf{7.042E-4}&	\textbf{2.477E-4} \\ \toprule\hline
Method & PDE  & Dim=$10^1$       & Dim=$10^2$   & Dim=$10^3$   & Dim=$10^4$   \\ \hline
FBSNN \cite{raissi2018forward} & Semilinear Heat & 3.872E-3	&8.323E-4	&7.705E-3	&1.426E-2
 \\ \hline
DeepBSDE \cite{han2018solving} & Semilinear Heat & 3.024E-3&	3.222E-3&	3.510E-3&	4.279E-3
 \\ \hline
DeepSplitting \cite{beck2021deep} & Semilinear Heat &  2.823E-3	&3.432E-3&	3.439E-3	&3.525E-3
 \\ \hline
SDGD (Ours) & Semilinear Heat & \textbf{1.052E-3}&	\textbf{5.263E-4}&	\textbf{6.910E-4}&	\textbf{1.598E-3}
  \\\hline
\toprule\hline
Method & PDE  & Dim=$10^1$       & Dim=$10^2$   & Dim=$10^3$   & Dim=$10^4$   \\ \hline
FBSNN \cite{raissi2018forward} & Sine-Gordon & 2.641E-3	&4.283E-3&2.343E-1&5.185E-1
 \\ \hline
DeepBSDE \cite{han2018solving} & Sine-Gordon & 3.217E-3&2.752E-3&2.181E-3&2.630E-3
 \\ \hline
DeepSplitting \cite{beck2021deep} & Sine-Gordon &  3.297E-3	&2.674E-3&	2.170E-3	&2.249E-3
\\ \hline
SDGD (Ours) & Sine-Gordon & \textbf{2.265E-3}	&\textbf{2.310E-3}&\textbf{1.674E-3}&\textbf{1.956E-3}
  \\\hline
\end{tabular}
\caption{Results of various methods on the Allen-Cahn, semilinear heat, and sine-Gordon PDEs with different dimensions. The evaluation metric is all relative $L_1$ error on one test point. SDGD (Ours) is the state-of-the-art (SOTA) method.}
\label{tab:nonlinear_more}
\end{table}

\subsection{Nonlinear Fokker-Planck PDEs and Comparison with Strong Baselines}
Next, we present further results comparing our method based on PINN and other strong non-PINN baselines \cite{beck2021deep, han2018solving, raissi2018forward} for several nonlinear PDEs without analytical solutions.
Thus, we adopt other methods' settings and evaluate the model on one test point, where we show that SDGD-PINNs are state-of-the-art (SOTA) models that still outperform other competitors.
Concretely, the following nonlinear PDEs are considered:
\begin{itemize}
\item Allen-Cahn equation.
\begin{equation}
\partial_t u(\bx,t) = \Delta u(\bx,t) + u(\bx,t) - u(\bx,t)^3, \quad (\bx,t) \in \mathbb{R}^d \times [0, 0.3],
\end{equation}
with the initial condition $u(\bx, t=0) = \arctan(\max_i \bx_i)$. We aim to approximate the solution's true value on the one test point $(\bx, t) = (0,\cdots,0,0.3)$.
\item Semilinear heat equation.
\begin{equation}
\partial_t u(\bx,t) = \Delta u(\bx,t) + \frac{1-u(\bx,t)^2}{1+u(\bx,t)^2}, \quad (\bx,t) \in \mathbb{R}^d \times [0, 0.3],
\end{equation}
with the initial condition $u(\bx, t=0) = 5 / (10 + 2\Vert\bx\Vert^2)$. We aim to approximate the solution's true value on the one test point $(\bx, t) = (0,\cdots,0,0.3)$.
\item Sine-Gordon equation.
\begin{equation}
\partial_t u(\bx,t) = \Delta u(\bx,t) + \sin\left(u(\bx,t) \right), \quad (\bx,t) \in \mathbb{R}^d \times [0, 0.3],
\end{equation}
with the initial condition $u(\bx, t=0) = 5 / (10 + 2\Vert\bx\Vert^2)$. We aim to approximate the solution's true value on the one test point $(\bx, t) = (0,\cdots,0,0.3)$.
\end{itemize}
The reference values of these PDEs on the test point are computed by the multilevel Picard method \cite{becker2020numerical, hutzenthaler2020overcoming} with sufficient theoretical accuracy. The exact solutions of these nonlinear PDEs are not solvable analytically and are highly nontrivial and nonseparable. The residual points of PINN are chosen along the trajectories of the stochastic processes that those PDEs correspond to:
\begin{equation}
t \sim \text{Unif}[0, 0.3], \bx \sim \mathcal{N}(0, (0.3-t) \cdot \boldsymbol{I}_{d\times d}).
\end{equation}
More training details are as follows. The model is a 4-layer multi-layer perceptron network with 1024 hidden units, which is trained via Adam \cite{kingma2014adam} for 10K epochs, with an initial learning rate 1e-3 which decays exponentially with exponent 0.9995. We discretize the time into a stepsize of 0.015 and select boundary and residual points from $|B|=100$ random SDE trajectories at each Adam epoch. We assign unity weight to the residual loss and 20 weight to the boundary loss, where in the latter, we fit the initial condition and its first-order derivative concerning $\bx$. We utilize Algorithm \ref{algo:2} with a minibatch of $|I|=|J|=10$ dimensions to randomly sample the second-order derivatives along various dimensions in the Laplacian in all the PDEs. We repeat our experiment 5 times with 5 independent random seeds.

The results for these three nonlinear PDEs are shown in Table \ref{tab:nonlinear_more}. PINN surpasses other methods in all dimensions and for all types of PDEs. Whether it is predicting the entire domain or the value at a single point, PINN is capable of handling both scenarios. This is attributed to the mesh-free nature of PINN and its powerful interpolation capabilities.

\subsection{SDGD Accelerates PINN's Adversarial Training in HJB PDEs}

In Wang et al. \cite{wang20222}, it was demonstrated that the class of Hamilton-Jacobi-Bellman (HJB) equation \cite{han2018solving,wang20222} could only be effectively solved using adversarial training, which approximates the $L^\infty$ loss. Specifically, they consider the HJB equation with linear-quadratic-Gaussian (LQG) control:
\begin{equation}
\begin{aligned}
&\partial_t u(\bx, t) + \Delta u(\bx, t) - \Vert \nabla_{\bx} u(\bx,t) \Vert^2 = 0, \quad \bx \in \mathbb{R}^d, t\in[0,1]\\
&u(\bx,T)=g(\bx),
\end{aligned}
\end{equation} 
where $g(\bx)$ is the terminal condition/cost to be chosen. We can use the exact solution for testing and benchmarking PINNs' performances:
\begin{equation}
u(\bx,t) = -\log\left(\int_{\mathbb{R}^d}(2\pi)^{-d/2}\exp(-\Vert \boldsymbol{y} \Vert^2/2)\exp(- g(\bx - \sqrt{2(1-t)}\boldsymbol{y}))d\boldsymbol{y}\right).
\end{equation}
We choose the following cost functions 
\begin{itemize}
\item Logarithm cost:
\begin{equation}
g(\bx) = \log\left(\frac{1+\Vert \bx \Vert^2}{2}\right)
\end{equation}
following many previous papers \cite{han2018solving,he2023learning,raissi2018forward, wang20222} to obtain the \textbf{HJB-Log} case. We demonstrate that SDGD is the state-of-the-art (SOTA)
high-dimensional PDE solver, where SDGD outperforms these previous methods on this classical benchmark PDE, especially the closely related PINN-based approach in \cite{he2023learning,wang20222}.
\item Rosenbrock cost:
\begin{equation}
g(\bx) = \log\left(\frac{1+\sum_{i=1}^{d-1}\left[c_{1,i}( \bx_{i} - \bx_{i+1})^2 + c_{2,i}\bx_{i+1}^2\right]}{2}\right),
\end{equation}
where $c_{1,i},c_{2,i} \sim \text{Unif}[0.5, 1.5]$. This nonconvex cost function will lead to a PDE exact solution that is anisotropic and asymmetric along dimensions, highly coupled and inseparable. Notable, in the exact solution, all pairs of variables, both $\bx_i$ and $\bx_{i+1}$, exhibit pairwise interactions and are highly coupled. The corresponding HJB PDE has an effective dimension of $d + 1$, i.e., it can not be reduced to low-dimensional subproblems. This case is called the \textbf{HJB-Rosenbrock} case.
\end{itemize}

The solution's integration in HJB PDE cannot be analytically solved, necessitating Monte Carlo integration for approximation. We use the relative $L_2$ error approximating $u$ as the evaluation metric for this equation.

In this example of the HJB equation with LQG control, we decompose the residual prediction of PINN along each dimension as follows:
\begin{equation}
\partial_t u(\bx, t) + \Delta u(\bx, t) - \mu \Vert \nabla_{\bx} u(\bx,t) \Vert^2 = \frac{1}{d}\sum_{i=1}^d \left\{\partial_t u(\bx, t)+ d\frac{\partial^2u}{\partial \bx_i^2}-\mu \Vert \nabla_{\bx} u(\bx,t) \Vert^2 \right\}.
\end{equation}

Despite its ability to maintain a low maximal memory cost during training, adversarial training approximating the $L^\infty$ loss is widely recognized for its slow and inefficient nature \cite{shafahi2019adversarial}, which poses challenges in applying it to high-dimensional HJB equations with LQG control. Adversarial training involves optimizing two loss functions: one for the PINN parameters, minimizing the loss through gradient descent, and another for the residual point coordinates, maximizing the PINN loss to approximate the $L^\infty$ loss. This adversarial minimax process, known as adversarial training, is computationally demanding, often requiring multiple rounds of optimization before the resulting residual points can effectively optimize the PINN parameters.

The current state-of-the-art (SOTA) research by Wang et al. \cite{wang20222} and He et al. \cite{he2023learning} has successfully scaled adversarial training to 250 dimensions. In this study, we use SDGD to enhance the scalability and efficiency of adversarial training of PINN in high-dimensional HJB equations.

The implementation details are given as follows. We use a four-layer PINN with 1024 hidden units, trained by an Adam optimizer \cite{kingma2014adam} with a learning
rate = 1e-3 at the beginning and decay linearly to zero. We modified the PINN structure
to make the terminal condition automatically satisfied \cite{lu2021physics}:
\begin{equation}
{u}^{\text{HJB-LQG}}_\theta(\bx, t) = u_\theta(\bx, t) (1-t)+ g(\bx),
\end{equation}
where $g(\bx)$ is the terminal condition.  Thus, we only need to focus on the residual loss.
We conduct training for a total of 10,000 epochs. In all dimensions, we employ Algorithm \ref{algo:3} with $|I|=100$ batches of PDE terms (dimensions) in the loss function for gradient descent and Algorithm \ref{algo:3} with $|I|=10$ batches of PDE terms (dimensions) in the loss function for adversarial training. For both our methods and the baseline methods proposed by Wang et al. \cite{wang20222} and He et al. \cite{he2023learning}, we randomly sample $|B|=100$ residual points per epoch, and 20K test points, both sampling from the distribution $t \sim \text{Unif[0,1]}, \bx \sim \mathcal{N}(0, \boldsymbol{I}_{d\times d})$. For generating the reference value for testing, we conduct a Monte Carlo simulation with 1e5 samples to estimate the integral in the exact solution. We repeat our experiment 5 times with 5 independent random seeds.

The results of adversarial training on the two HJB equations are presented in Table \ref{tab:adv}. For instance, in the 250D HJB-LQG case, achieving a relative $L_2$ error of 1e-2 using \cite{wang20222} requires a training time of 1 day. In contrast, our approach achieves a superior $L_2$ error in just 3.5 hours. In the 1000D case, full batch GD's adversarial training is exceedingly slow, demanding over 5 days of training time. However, our method effectively mitigates this issue and attains a relative $L_2$ error of 5.852E-3 in 217 minutes. Moreover, our approach enables adversarial training even in higher-dimensional HJB equations, specifically 10,000 and 100,000 dimensions, with resulting relative $L_2$ errors of approximately 5E-2. For the HJB-Rosenbrock case, our method is also consistently better than others, demonstrating SDGD's strong capability to deal with a complicated, inseparable, effectively high-dimensional, non-convex, and highly coupled solution.

\begin{table}[htbp]
\centering
\begin{tabular}{|c|cc|cc|cc|}
\hline
\multicolumn{7}{|c|}{HJB-Log Results for $u$}\\\hline
        & \multicolumn{2}{c|}{Wang et al. \cite{wang20222}}                            & \multicolumn{2}{c|}{He et al. \cite{he2023learning}}                              & \multicolumn{2}{c|}{SDGD (Ours)}                           \\ \hline
Dim     & \multicolumn{1}{c|}{Time}                & Rel. $L_2$ Error & \multicolumn{1}{c|}{Time}                & Rel. $L_2$ Error & \multicolumn{1}{c|}{Time}        & Rel. $L_2$ Error \\ \hline
250     & \multicolumn{1}{c|}{38 hours}            & 1.182E-2         & \multicolumn{1}{c|}{12 hours}            & 1.370E-2         & \multicolumn{1}{c|}{210 minutes} & 6.147E-3         \\ \hline
1,000   & \multicolumn{1}{c|}{\textgreater 5 days} & N.A.             & \multicolumn{1}{c|}{\textgreater 5 days} & N.A.             & \multicolumn{1}{c|}{217 minutes} & 5.852E-3         \\ \hline
10,000  & \multicolumn{1}{c|}{OOM}                 & N.A.             & \multicolumn{1}{c|}{OOM}                 & N.A.             & \multicolumn{1}{c|}{481 minutes} & 4.926E-2         \\ \hline
100,000 & \multicolumn{1}{c|}{OOM}                 & N.A.             & \multicolumn{1}{c|}{OOM}                 & N.A.             & \multicolumn{1}{c|}{855 minutes} & 4.852E-2         \\ \hline
\multicolumn{7}{|c|}{HJB-Rosenbrock Results for $u$}\\\hline
        & \multicolumn{2}{c|}{Wang et al. \cite{wang20222}}                            & \multicolumn{2}{c|}{He et al. \cite{he2023learning}}                              & \multicolumn{2}{c|}{SDGD (Ours)}                           \\ \hline
Dim     & \multicolumn{1}{c|}{Time}                & Rel. $L_2$ Error & \multicolumn{1}{c|}{Time}                & Rel. $L_2$ Error & \multicolumn{1}{c|}{Time}        & Rel. $L_2$ Error \\ \hline
250     & \multicolumn{1}{c|}{38 hours}            &    1.207E-2    & \multicolumn{1}{c|}{12 hours}            &     1.413E-2  & \multicolumn{1}{c|}{210 minutes} &    5.419E-3      \\ \hline
1,000   & \multicolumn{1}{c|}{\textgreater 5 days} & N.A.             & \multicolumn{1}{c|}{\textgreater 5 days} & N.A.             & \multicolumn{1}{c|}{217 minutes} &   4.153E-3\\ \hline
10,000  & \multicolumn{1}{c|}{OOM}                 & N.A.             & \multicolumn{1}{c|}{OOM}                 & N.A.             & \multicolumn{1}{c|}{481 minutes} &   4.168E-2  \\ \hline
100,000 & \multicolumn{1}{c|}{OOM}                 & N.A.             & \multicolumn{1}{c|}{OOM}                 & N.A.             & \multicolumn{1}{c|}{855 minutes} &   4.091E-2  \\ \hline
\end{tabular}
\caption{Relative $L_2$ error results, running time across different dimensions for the baselines and ours, for two HJB equations requiring adversarial training of PINNs. In all dimensions, our method significantly outperforms the other two baselines in terms of accuracy and speed. The baselines, employing full batch gradient descent, experience slower speed and increased memory requirements in higher-dimensional scenarios. In contrast, our approach utilizes Algorithm \ref{algo:3}, which employs stochastic gradient descent over PDE terms, resulting in faster computation and relatively lower memory consumption.}
\label{tab:adv}
\end{table}

Furthermore, although the HJB equation is also parabolic and non-PINN traditional methods like \cite{beck2021deep,han2018solving,han2017deep} can be used, we test on 20,000 points. SDGD with PINN trains a surrogate model to predict these 20,000 points after one training process. However, these traditional methods can only predict one point after one training instance, so they require 20,000 rounds of separate inferences, which is too costly.

\subsection{Schr\"{o}dinger Equation}
Here, we scale up the Tensor Neural Networks (TNNs) \cite{wang2022tensor,wang2022solving} for solving very high-dimensional Schr\"{o}dinger equations. Due to its unique characteristics, the Schr\"{o}dinger equation cannot be addressed by the traditional PINN framework and requires a specialized TNN network structure and corresponding loss function for high-dimensional eigenvalue problems. However, our method can still be flexibly integrated into this framework for scaling up Schr\"{o}dinger equations, demonstrating the versatility of our approach. In contrast, other methods specifically designed for high-dimensional equations exhibit much greater limitations.

\subsubsection{Tensor Neural Network (TNN) for Solving Schr\"{o}dinger Equation}
In general, the Schr\"{o}dinger equation can be viewed as a high-dimensional eigenvalue problem.
\begin{equation}\label{eq:schrodinger}
\begin{aligned}
-\Delta u(\bx) + v(\bx) u(\bx) &= \lambda u(\bx), \quad \bx
\in \Omega,\\
u(\bx) &= 0, \quad \bx \in \partial \Omega,
\end{aligned}
\end{equation}
where $ \Omega = \Omega_1 \times \Omega_2 \times \cdots \times \Omega_d \subset \mathbb{R}^d$ and each $\Omega_i = (a_i, b_i)$, $i = 1, \cdots, d$ is a bounded interval in $\mathbb{R}$, $v(\bx)$ is a known potential function, $u(\bx)$ is the unknown solution and $\lambda$ is the unknown eigenvalue of the high-dimensional eigenvalue problem. This problem can be addressed via the variational principle:
\begin{align}
\lambda = \min_{\Phi(\bx)}\frac{\int_\Omega |\nabla \Phi(\bx)|^2d\bx + \int_\Omega v(\bx)\Phi(\bx)^2d\bx}{\int_\Omega \Phi(\bx)^2d\bx}.
\end{align}
The integrals are usually computed via the expensive quadrature rule, e.g., Gaussian quadrature:
\begin{equation}
\int_\Omega \Phi(\bx) d\bx \approx \sum_{n \in \mathcal{N}} w^{(n)}\Phi(\bx^{(n)}),
\end{equation}
where $\mathcal{N}$ is the index set for the quadrature points. The size $\mathcal{N}$ grows exponentially as the dimension increases, incurring the curse of dimensionality.
Hence, traditional methods for solving Schr\"{o}dinger equations suffer from the curse of dimensionality due to the exponential growth of the grid size in the quadrature for numerical integral in the equation's variational form. So, traditional methods cannot be used even in the lowest 100-dimensional case in our experiment.
To overcome this issue, the Tensor Neural Network (TNN) adopts a separable function network structure to reduce the computational cost associated with integration:
\begin{align}
\Phi(\bx;\theta) = \sum_{j=1}^p \prod_{i=1}^d\phi_{j}(\bx_i;\theta_i),
\end{align}
where $\bx_i$ is the $i$th dimension of the input point $\bx$, and $\phi_{j}(\bx_i;\theta_i)$ is the $j$th output of the sub-wavefunction network parameterized by $\theta_i$, and $p$ is the predefined rank of the TNN model.
The integral by quadrature rule of the TNN can be calculated efficiently:
\begin{equation}
\int_\Omega \Phi(\bx; \theta) d\bx = \sum_{j=1}^p  \int_\Omega \prod_{i=1}^d\phi_{j}(\bx_i;\theta_i) d\bx = \sum_{j=1}^p  \prod_{i=1}^d\int_{\Omega_i} \phi_{j}(\bx_i;\theta_i) d\bx_i \approx \sum_{j=1}^p  \prod_{i=1}^d\sum_{n_i \in \mathcal{N}_i} w^{(n_i)}\phi_j(x^{(n_i)};\theta_i),
\end{equation}
where $\mathcal{N}_i$ is the index set for the $i$th dimensional numerical integral. Thus, the numerical integrals can be computed along each axes separately to reduce the exponential cost to linear cost.
We refer the readers to \cite{wang2022tensor, wang2022solving} for more details.
With the efficient quadrature achieved by TNN's separable structure, the loss function to be minimized is
\begin{equation}
\min_\theta\ell(\theta) = \frac{ \sum_{n \in \mathcal{N}} w^{(n)}|\nabla \Phi(\bx^{(n)}; \theta)|^2 +  \sum_{n \in \mathcal{N}} w^{(n)}v(\bx^{(n)})\Phi(\bx^{(n)}; \theta)^2}{ \sum_{n \in \mathcal{N}} w^{(n)} \Phi(\bx^{(n)}; \theta)^2},
\end{equation}
where $\Phi(\bx; \theta)$ is the neural network wave function with trainable parameters $\theta$.

The primary computational bottleneck of TNNs for this problem lies in the first-order derivatives in the loss function, while other terms only contain zero-order terms with lower computational and memory costs.
Since TNN is a separable architecture, the $d$-dimensional Schrödinger equation requires computing the first-order derivatives and performing quadrature integration for each of the $d$ subnetworks. Consequently, the computational cost scales linearly with the dimension. In comparison to previous examples of PINN solving PDEs, TNN incurs larger memory consumption for first-order derivatives due to the non-sharing of network parameters across different dimensions, i.e., $d$ independent first-order derivatives of $d$ distinct subnetworks are required, resulting in increased computational requirements.

Mathematically, the computational and memory bottleneck is the gradient with respect to the first-order term of the neural network wave function $\Phi(\bx; \theta)$
\begin{equation}
\text{grad}(\theta) = \sum_{n \in \mathcal{N}} w^{(n)}\frac{\partial}{\partial\theta}|\nabla \Phi_\theta(\bx^{(n)})|^2 = \sum_{n \in \mathcal{N}} w^{(n)}\left(\sum_{i=1}^d\frac{\partial}{\partial\theta}\left|\frac{\partial}{\partial \bx_i} \Phi_\theta(\bx^{(n)})\right|^2\right).
\end{equation}
By employing the proposed SDGD approach, we can simplify the entire gradient computation:
\begin{equation}
\text{grad}_I(\theta) = \sum_{n \in \mathcal{N}} w^{(n)}\left(\frac{d}{|I|}\sum_{i \in I}\frac{\partial}{\partial\theta}\left|\frac{\partial}{\partial \bx_i} \Phi_\theta(\bx^{(n)})\right|^2\right),
\end{equation}
where $I \subset \{1,2,\cdots,d\}$ is a random index set.
Therefore, we can sample the entire gradient at the PDE term level to reduce memory and computational costs, and the resulting stochastic gradient is unbiased, i.e., $\mathbb{E}_I[\text{grad}_I(\theta)] =\text{grad}(\theta)$ ensuring the convergence of our method. The gradient accumulation can also be done for large-scale problems by our method; see Section \ref{sec:GAPC} for details.

It is important to note that our method is more general compared to the traditional approach of SGD solely based on quadrature points. In particular, in extremely high-dimensional cases, even computing the gradient for a single quadrature point may lead to an out-of-memory (OOM) error. This is because the traditional approach requires a minimum batch size of $d$ PDE terms and 1 quadrature point, where $d$ is the dimensionality of the PDE. However, our method further reduces the computational load per batch to only 1 PDE term and 1 quadrature point.

\subsubsection{Experimental Setup}
We consider the \textbf{Coupled Quantum Harmonic Oscillator (CQHO)} potential function: $v(\bx) = \sum_{i=1}^d \bx_i^2 - \sum_{i=1}^{d-1} \bx_i \bx_{i+1}$, where all pairs of variables, both $\bx_i$ and $\bx_{i+1}$, exhibit pairwise
interactions and are highly coupled. The original problem is defined on an infinite interval. Therefore, we truncate the entire domain to $[-5,5]^d$. The exact smallest eigenvalue is $\lambda = \sum_{i=1}^d\sqrt{1-\cos\left(\frac{i\pi}{d+1}\right)}$. We use the same model structure and hyperparameters as Wang et al. \cite{wang2022tensor}. For this problem, we test the effectiveness of the Gradient Accumulation method in Section \ref{sec:GAPC} based on our further decomposition of PINN's gradient.

The test metric we report is the $L_1$ relative error of the minimum eigenvalue, given by $\frac{| \lambda_{\text{true}} - \lambda_{\text{approx}} |}{| \lambda_{\text{true}} |}$. Due to the diminishing nature of the eigenfunctions in high dimensions, we only report the accuracy of the eigenvalues. It is important to note that the eigenvalues carry physical significance as they represent the ground state energy.

\begin{figure}
\centering
\includegraphics[scale=0.57]{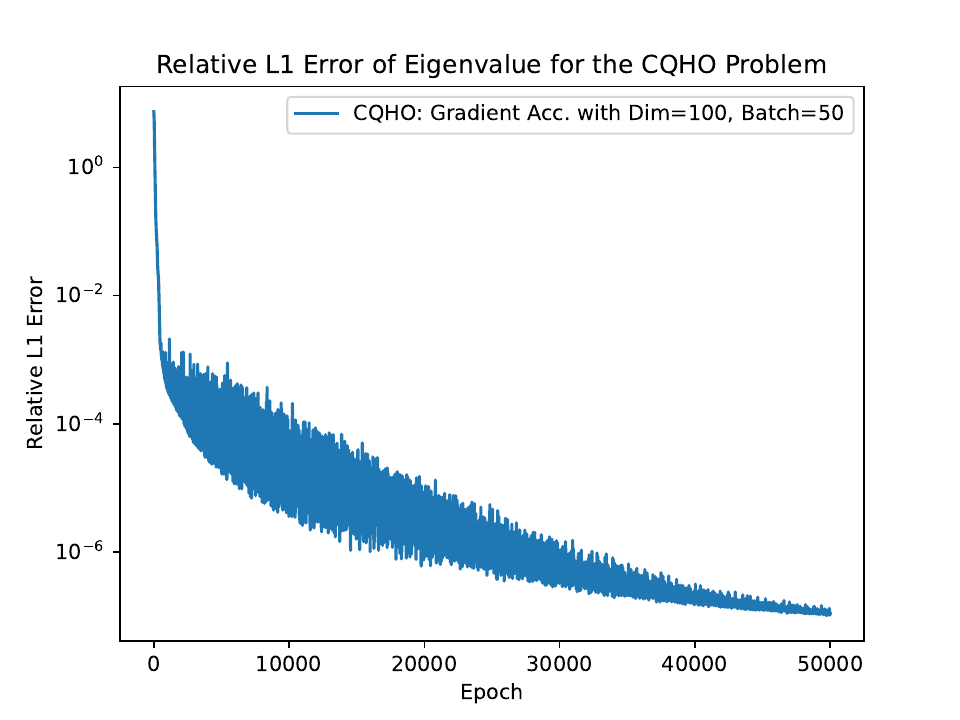}
\includegraphics[scale=0.57]{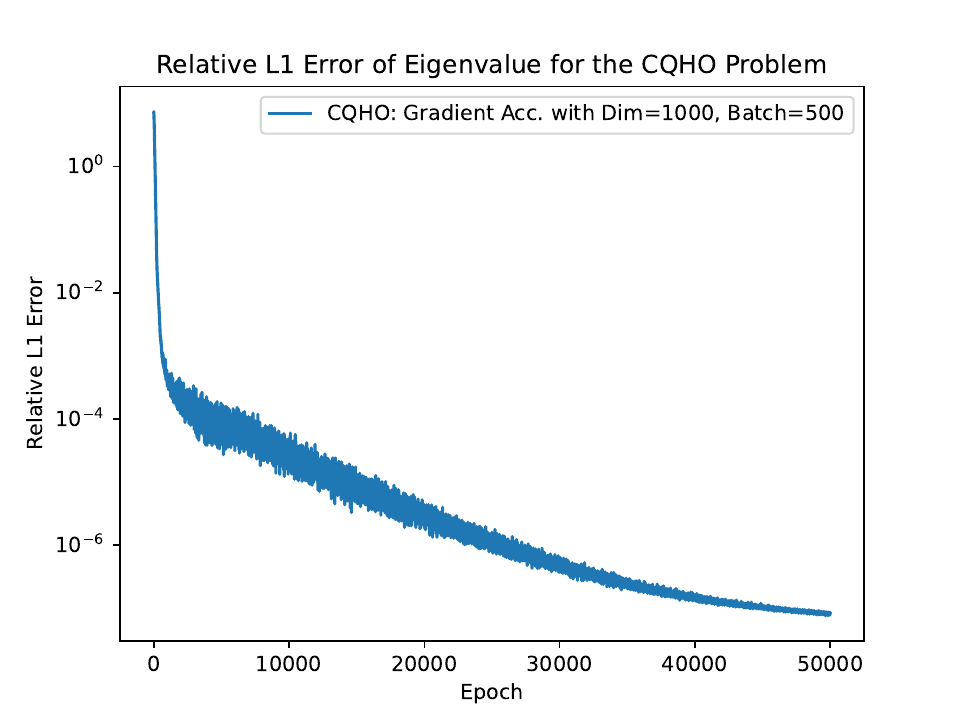}
\includegraphics[scale=0.57]{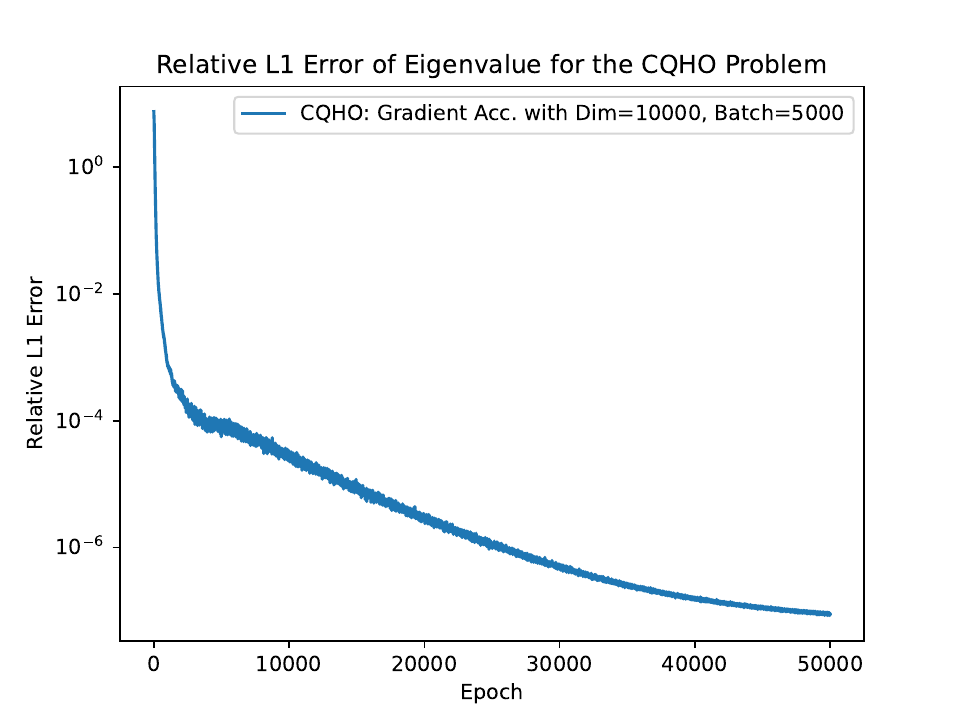}
\includegraphics[scale=0.57]{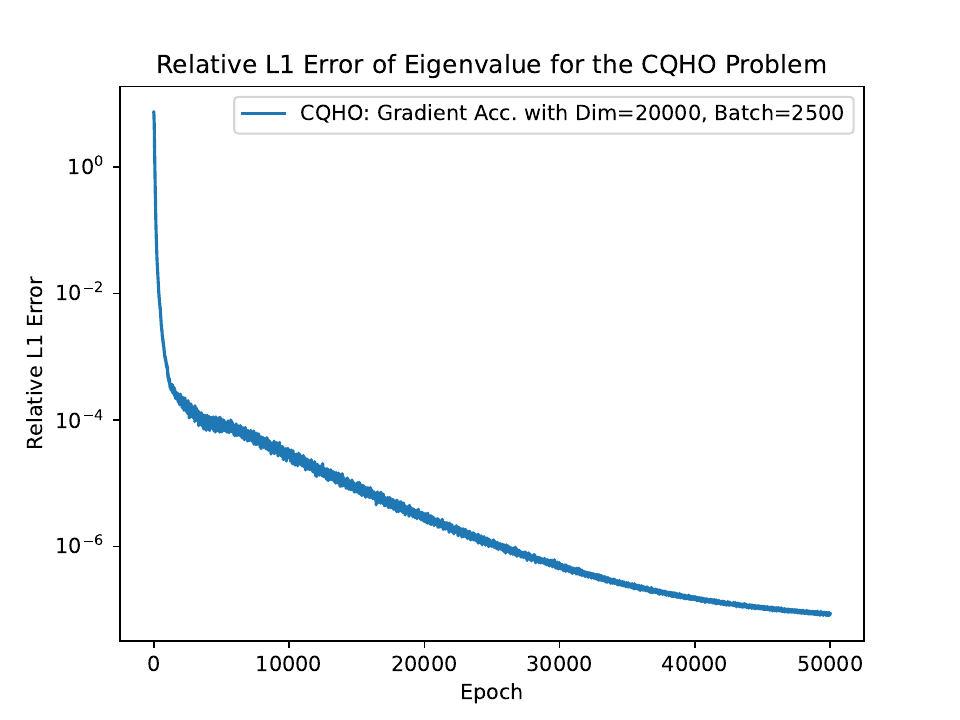}
\caption{Coupled quantum harmonic oscillator (CQHO) problem: convergence curve of our algorithm (gradient accumulation) under different dimensions.}
\label{fig:qho}
\end{figure}

\subsubsection{Experimental Results}
On the CQHO problem, our gradient accumulation method achieves training in $2*10^4$ dimensions within the limited GPU memory in Figure \ref{fig:qho}. In all dimensions, SDGD achieves a stable relative error of less than 1e-6. Note that gradient accumulation performs theoretically the same as full batch GD, so we do not compare their results, but the former can save GPU memory, thus scaling up PDEs to higher dimensions. These results demonstrate the scalability of our method for PINN and PDE problems and highlight its generality, making it applicable to various physics-informed machine-learning problems.

During the early stages of optimization, we observe quick convergence of TNNs. Since the model parameters are initialized randomly, there is more room for improvement. At this point, the gradients of the loss function with respect to the parameters can be relatively large. Larger gradients result in more significant updates to the parameters, leading to faster convergence.

Despite the satisfactory convergence results, these problems demonstrate one shortcoming of our proposed approach. Our method is primarily designed to address the computational bottleneck arising from high-dimensional differentiations in PDEs. However, if the problem itself is large-scale and does not stem from differential equations, our method cannot be directly applied. Specifically, in the case of the Schrödinger equation, due to the separable nature of its network structure, we employ a separate neural network for each dimension. This results in significant memory consumption, even without involving differentials. Additionally, in this problem, we cannot sample the points of integration / residual points because there is an integral term in the denominator. If we sample the integral in the denominator, the resulting stochastic gradient will be biased, leading to poor convergence. Therefore, for such problems, designing an improved unbiased stochastic gradient remains an open question.

\section{Conclusions}
CoD has been an open problem for many decades, but recent progress has been made by several research groups using deep learning methods for specific classes of PDEs \cite{beck2021deep, han2018solving, raissi2018forward}. Herein, we proposed a general method based on the mesh-free PINNs method by introducing a new type of stochastic dimension gradient descent or SDGD as a generalization of the largely successful SGD method over subsets of the training set. In particular, we claim the following advantages of the SDGD-PINNs method over other related methods:

\textbf{Generality to Arbitrary PDEs}. Most existing approaches \cite{beck2021deep, han2018solving, raissi2018forward} are restricted to specific forms of parabolic PDEs. They leverage the connection between parabolic PDEs and stochastic processes, employing Monte Carlo simulations of stochastic processes to train their models. Consequently, their methods are limited to a subset of PDEs. In contrast, our approach is based on PINNs capable of handling arbitrary PDEs. Notably, we can address challenging cases such as the wave equation, biharmonic equation, Schrödinger equation, and other similar examples, while \cite{beck2021deep, han2018solving, raissi2018forward} cannot.

\textbf{Prediction on the Entire Domain}. Furthermore, existing methods \cite{beck2021deep, han2018solving, raissi2018forward} can only predict the values of PDE solutions at a single test point during each training instance. This limitation arises from the necessity of setting the starting point of the stochastic process (i.e., the test point for the PDE solution) before conducting Monte Carlo simulations. Consequently, the computational cost of their methods \cite{beck2021deep, han2018solving, raissi2018forward}  increases greatly as the number of test points grows. In contrast, our approach, based on PINNs, allows for predictions across the entire domain in a single training instance, achieving a ``once for all" capability.

\textbf{Dependency on the Mesh}. Moreover, precisely because these methods \cite{beck2021deep, han2018solving, raissi2018forward} require Monte Carlo simulations of stochastic processes, they necessitate temporal discretization. This introduces additional parameters and requires a sufficiently high discretization precision to obtain accurate solutions. For PDE problems that span long durations, these methods undoubtedly suffer from the burden of temporal discretization. However, our PINN-based approach can handle long-time PDEs effortlessly. Since PINNs are mesh-free methods, the training points can be randomly sampled. The interpolation capability of PINNs ensures their generalization performance on long-time PDE problems.

\textbf{PINNs can also be adapted for prediction at one point}. Although PINNs offer advantages such as being mesh-free and capable of predicting the entire domain, they still perform remarkably well in settings where other methods focus on, i.e., single-point prediction. Specifically, if our interest lies in the value of the PDE at a single point, we can exploit the relationship between the PDE and stochastic processes. This single point corresponds to the initial point of the stochastic process, and the PDE corresponds to a specific stochastic process, such as the heat equation corresponding to simple Brownian motion. In this case, we only need to use the snapshots obtained from the trajectory of this stochastic process as the residual points for PINN.

We have proved the necessary theorems that show SDGD leads to an unbiased estimator and that it converges under mild assumptions, similar to the vanilla SGD over collocation points and mini-batches. SDGD can also be used in physics-informed neural operators such as the DeepONet \cite{lu2019deeponet,goswami2022physicsinformed} and the FNO \cite{li2021fourier,li2021physics} and can be extended to general regression and possibly classification problems in very high dimensions.

The potential limitation of SDGD is that it requires a relatively larger batch size for the dimension when dealing with extremely high-dimensional PDEs. A too-small batch size may incur high stochastic gradient variance that slows down the convergence. In addition, for PDEs not following the boundary+residual forms, like the Schrödinger equation, a special loss function is generally required, which may cause the core memory bottleneck to not be in the dimension that SDGD focuses on. This causes the minimum memory that SDGD can obtain to grow faster with dimensions.

In summary, the new SDGD-PINN framework is a paradigm shift in the way we solve high-dimensional PDEs as it can tackle arbitrary nonlinear PDEs, of any dimension providing high accuracy at extremely fast speeds.

\section*{Acknowledgments}
The work of ZH and KK was supported by the Singapore Ministry of Education Academic Research Fund Tier 1 grant (Grant number: T1 251RES2207). 
The work of KS and GEK was supported by the MURI-AFOSR FA9550-20-1-0358 projects and by the DOE SEA-CROGS project (DE-SC0023191). GEK was also supported by the ONR Vannevar Bush Faculty Fellowship (N00014-22-1-2795).

\appendix

\section{Proof}
\subsection{Proof of Theorem \ref{thm:unbiased}}\label{appendix:unbiased}
The unbiasedness of the stochastic gradient in Algorithm \ref{algo:1} can be derived as follows:
\begin{equation}
\begin{aligned}
\mathbb{E}_I \left[\text{grad}_I(\theta)\right] &= \left(\sum_{i=1}^{N_\mathcal{L}}\mathcal{L}_i u_\theta(\bx) - R(\bx)\right)\mathbb{E}_I \left[\frac{N_{\mathcal{L}}}{|I|}{\sum_{i \in I}} \frac{\partial}{\partial\theta}\mathcal{L}_iu_\theta(\bx)\right]\\
&=\left(\sum_{i=1}^{N_\mathcal{L}}\mathcal{L}_i u_\theta(\bx) - R(\bx)\right)\left[{\sum_{i =1}^{N_{\mathcal{L}}}} \frac{\partial}{\partial\theta}\mathcal{L}_iu_\theta(\bx)\right]\\
&= \text{grad}(\theta).
\end{aligned}
\end{equation}
The unbiasedness of the stochastic gradient in Algorithm \ref{algo:2} can be derived as follows:
\begin{equation}
\begin{aligned}
\mathbb{E}_{I,J} \left[\text{grad}_{I,J}(\theta) \right]&=\left(\mathbb{E}_{J}\left[\frac{N_{\mathcal{L}}}{|J|}\sum_{j \in J} \mathcal{L}_j u_\theta(\bx)\right] - R(\bx)\right)\mathbb{E}_I\left[\frac{N_{\mathcal{L}}}{|I|}\sum_{i\in I}\frac{\partial}{\partial\theta} \mathcal{L}_iu_\theta(\bx)\right]\\
&=\left(\sum_{i=1}^{N_\mathcal{L}}\mathcal{L}_i u_\theta(\bx) - R(\bx)\right)\left[{\sum_{i =1}^{N_{\mathcal{L}}}} \frac{\partial}{\partial\theta}\mathcal{L}_iu_\theta(\bx)\right]\\
&= \text{grad}(\theta).
\end{aligned}
\end{equation}

\subsection{Proof of Theorem \ref{thm:variance_1}}\label{appendix:variance_1}
\begin{lemma}\label{lemma:1}
Suppose that we have $N$ fixed numbers $a_1, a_2,\cdots,a_N$ and we choose $k$ random numbers $\{X_i\}_{i=1}^{k}$ from them, i.e., each $X_i$ is a random variable and $X_i = a_n$ with probability $1 / N$ for all $n \in \{1,2,\cdots,N\}$, and $X_i$ are independent. Then the variance of the unbiased estimator $\sum_{i=1}^k X_i / k$ for $\overline{a} = \sum_{n=1}^N a_n$ is $\frac{1}{kn}\sum_{i=1}^n(a_i - \Bar{a})^2$.
\end{lemma}
\begin{proof}[Proof of Lemma \ref{lemma:1}]
Since the samples $X_i$ are i.i.d.,
\begin{equation}
\begin{aligned}
\mathbb{V}\left[\frac{\sum_{i=1}^k X_i}{k}\right]
&= \frac{1}{k}\mathbb{V}[X_i] = \frac{1}{kn}\sum_{i=1}^n(a_i - \Bar{a})^2
\end{aligned}
\end{equation}
\end{proof}
\begin{proof}[Proof of Theorem \ref{thm:variance_1}]
We assume that the total full batch is with $N_r$ residual points $\{\bx_i\}_{i=1}^{N_r}$ and $N_{\mathcal{L}}$ PDE terms $\mathcal{L} = \sum_{i=1}^{N_{\mathcal{L}}}\mathcal{L}_i$, then the residual loss of the PINN is
\begin{equation}
\frac{1}{2N_rN^2_{\mathcal{L}}}\sum_{i=1}^{N_r}\left(\mathcal{L}u_\theta(\bx_i) - R(\bx_i)\right)^2,
\end{equation}
where we normalize over both the number of residual points and the number of PDE terms. The full batch gradient is:
\begin{equation}
\begin{aligned}
g(\theta) &= \frac{1}{N_rN^2_{\mathcal{L}}}\sum_{i=1}^{N_r}\left(\mathcal{L}u_\theta(\bx_i) - R(\bx_i)\right)\frac{\partial}{\partial \theta}\mathcal{L}u_\theta(\bx_i)\\
&= \frac{1}{N_rN^2_{\mathcal{L}}}\sum_{i=1}^{N_r}\left(\mathcal{L}u_\theta(\bx_i) - R(\bx_i)\right)\left(\sum_{j=1}^{N_{\mathcal{L}}}\frac{\partial}{\partial \theta}\mathcal{L}_ju_\theta(\bx_i)\right)\\
&= \frac{1}{N_rN^2_{\mathcal{L}}}\sum_{i=1}^{N_r}\sum_{j=1}^{N_{\mathcal{L}}}\left(\mathcal{L}u_\theta(\bx_i) - R(\bx_i)\right)\left(\frac{\partial}{\partial \theta}\mathcal{L}_ju_\theta(\bx_i)\right)\\
&= \frac{1}{N_rN^2_{\mathcal{L}}}\sum_{i=1}^{N_r}\left(\left(\sum_{j=1}^{N_{\mathcal{L}}}\mathcal{L}_ju_\theta(\bx_i)\right) - R(\bx_i)\right)\left(\sum_{j=1}^{N_{\mathcal{L}}}\frac{\partial}{\partial \theta}\mathcal{L}_ju_\theta(\bx_i)\right).
\end{aligned}
\end{equation}
From the viewpoint of Algorithm \ref{algo:1}, the full batch gradient is the mean of $N_rN_{\mathcal{L}}$ terms:
\begin{equation}
g(\theta) = \frac{1}{N_rN^2_{\mathcal{L}}}\sum_{i=1}^{N_r}\sum_{j=1}^{N_{\mathcal{L}}}\left(\mathcal{L}u_\theta(\bx_i) - R(\bx_i)\right)\left(\frac{\partial}{\partial \theta}\mathcal{L}_ju_\theta(\bx_i)\right):= \frac{1}{N_rN^2_{\mathcal{L}}}\sum_{i=1}^{N_r}\sum_{j=1}^{N_{\mathcal{L}}}g_{i,j}(\theta),
\end{equation}
where we denote 
\begin{equation}
g_{i,j}(\theta) = \left(\mathcal{L}u_\theta(\bx_i) - R(\bx_i)\right)\left(\frac{\partial}{\partial \theta}\mathcal{L}_ju_\theta(\bx_i)\right).
\end{equation}
The stochastic gradient generated by Algorithm \ref{algo:1} is given by
\begin{equation}
g_{B, J}(\theta) = \frac{1}{|B||J|N_{\mathcal{L}}}\sum_{i \in B}\left(\mathcal{L}u_\theta(\bx_i) - R(\bx_i)\right)\left(\sum_{j \in J}\frac{\partial}{\partial \theta}\mathcal{L}_ju_\theta(\bx_i)\right) = \frac{1}{|B||J|N_{\mathcal{L}}}\sum_{i \in B}\sum_{j \in J}g_{i,j}(\theta).
\end{equation}
We have the variance
\begin{equation}
\mathbb{V}_{B, J}[g_{B, J}(\theta)] =\mathbb{E}_{B, J}[\left(g_{B, J}(\theta) - g(\theta)\right)^2]= \mathbb{E}_{B, J}[g_{B, J}(\theta)^2] - \mathbb{E}_{B, J}[g(\theta)^2] = \mathbb{E}_{B, J}[g_{B, J}(\theta)^2] - g(\theta)^2.
\end{equation}
From high-level perspective, $\mathbb{V}_{B, J}[g_{B, J}(\theta)]$ should be a function related to $|B|$ and $|J|$. Thus, under the constraint that $|B| \cdot |J|$ is the same for different SGD schemes, we can choose $B$ and $J$ properly to minimize the variance of SGD to accelerate convergence. 
\begin{equation}
\begin{aligned}
\mathbb{E}_{B, J}[g_{B, J}(\theta)^2] &=\frac{1}{|B|^2|J|^2N_{\mathcal{L}}^2}\mathbb{E}_{B, J}\left[\sum_{i \in B}\sum_{j \in J}\left(\mathcal{L}u_\theta(\bx_i) - R(\bx_i)\right)\left(\frac{\partial}{\partial \theta}\mathcal{L}_ju_\theta(\bx_i)\right)\right]^2\\
&= \frac{1}{|B|^2|J|^2N_{\mathcal{L}}^2}\mathbb{E}_{B, J}\Bigg[\sum_{i,i' \in B}\sum_{j,j' \in J}\left(\mathcal{L}u_\theta(\bx_i) - R(\bx_i)\right)\left(\frac{\partial}{\partial \theta}\mathcal{L}_ju_\theta(\bx_i)\right)\cdot\\
&\quad \left(\mathcal{L}u_\theta(\bx_{i'}) - R(\bx_{i'})\right)\left(\frac{\partial}{\partial \theta}\mathcal{L}_{j'}u_\theta(\bx_{i'})\right)\Bigg].
\end{aligned}
\end{equation}
The entire expectation can be decomposed into four parts (1) $i=i', j=j'$, (2) $i \neq i', j=j'$, (3) $i=i', j\neq j'$, (4) $i\neq i', j\neq j'$.
For the first part
\begin{equation}
\begin{aligned}
&\quad\frac{1}{|B|^2|J|^2N_{\mathcal{L}}^2}\mathbb{E}_{B, J}\Bigg[\sum_{i\in B}\sum_{j\in J}\left(\mathcal{L}u_\theta(\bx_i) - R(\bx_i)\right)^2\left(\frac{\partial}{\partial \theta}\mathcal{L}_ju_\theta(\bx_i)\right)^2\Bigg]\\
&=\frac{1}{|B||J|N_{\mathcal{L}}^2}\mathbb{E}_{i,j}\Bigg[\left(\mathcal{L}u_\theta(\bx_i) - R(\bx_i)\right)^2\left(\frac{\partial}{\partial \theta}\mathcal{L}_ju_\theta(\bx_i)\right)^2\Bigg]\\
&=\frac{1}{|B||J|N_{\mathcal{L}}^2}\mathbb{E}_{i,j}[g_{i,j}(\theta)^2]\\
&= \frac{1}{|B||J|N_rN_{\mathcal{L}}^3}\sum_{i=1}^{N_r}\sum_{j=1}^{N_{\mathcal{L}}} g_{i,j}(\theta)^2\\
&= \frac{C_1}{|B||J|},
\end{aligned}
\end{equation}
where we denote $C_1 = \frac{1}{N_rN_{\mathcal{L}}^3}\sum_{i=1}^{N_r}\sum_{j=1}^{N_{\mathcal{L}}} g_{i,j}(\theta)^2$ which is independent of $B, J$.

In the second case, since the $i$th sample and the $i'$th sample are independent for $i \neq i'$, we have
\begin{equation}
\begin{aligned}
&\quad\frac{1}{|B|^2|J|^2N_{\mathcal{L}}^2}\mathbb{E}_{B, J}\Bigg[\sum_{i\neq i'\in B}\sum_{j\in J}\left(\mathcal{L}u_\theta(\bx_i) - R(\bx_i)\right)\left(\mathcal{L}u_\theta(\bx_{i'}) - R(\bx_{i'})\right)\left(\frac{\partial}{\partial \theta}\mathcal{L}_ju_\theta(\bx_i)\right)\left(\frac{\partial}{\partial \theta}\mathcal{L}_{j}u_\theta(\bx_{i'})\right)\Bigg]\\
&=\frac{|B|-1}{|B||J|N_{\mathcal{L}}^2}\mathbb{E}_{i,i',j}\Bigg[\left(\mathcal{L}u_\theta(\bx_i) - R(\bx_i)\right)\left(\mathcal{L}u_\theta(\bx_{i'}) - R(\bx_{i'})\right)\left(\frac{\partial}{\partial \theta}\mathcal{L}_ju_\theta(\bx_i)\right)\left(\frac{\partial}{\partial \theta}\mathcal{L}_ju_\theta(\bx_{i'})\right)\Bigg]\\
&=C_2\cdot \frac{|B|-1}{|B||J|}.
\end{aligned}
\end{equation}
In the third case, due to the same reason,
\begin{equation}
\begin{aligned}
&\quad\frac{1}{|B|^2|J|^2N_{\mathcal{L}}^2}\mathbb{E}_{B, J}\Bigg[\sum_{i\in B}\sum_{j\neq j'\in J}\left(\mathcal{L}u_\theta(\bx_i) - R(\bx_i)\right)^2\left(\frac{\partial}{\partial \theta}\mathcal{L}_ju_\theta(\bx_i)\right)\left(\frac{\partial}{\partial \theta}\mathcal{L}_{j'}u_\theta(\bx_i)\right)\Bigg]\\
&=\frac{|J|-1}{|B||J|N_{\mathcal{L}}^2}\mathbb{E}_{i,j,j'}\Bigg[\left(\mathcal{L}u_\theta(\bx_i) - R(\bx_i)\right)^2\left(\frac{\partial}{\partial \theta}\mathcal{L}_ju_\theta(\bx_i)\right)\left(\frac{\partial}{\partial \theta}\mathcal{L}_{j'}u_\theta(\bx_i)\right)\Bigg]\\
&=C_3\cdot \frac{|J|-1}{|B||J|}.
\end{aligned}
\end{equation}
In the last case,
\begin{equation}
\begin{aligned}
&\quad\frac{1}{|B|^2|J|^2N_{\mathcal{L}}^2}\mathbb{E}_{B, J}\Bigg[\sum_{i\neq i'\in B}\sum_{j\neq j'\in J}\left(\mathcal{L}u_\theta(\bx_i) - R(\bx_i)\right)\left(\mathcal{L}u_\theta(\bx_{i'}) - R(\bx_{i'})\right)\left(\frac{\partial}{\partial \theta}\mathcal{L}_ju_\theta(\bx_i)\right)\left(\frac{\partial}{\partial \theta}\mathcal{L}_{j'}u_\theta(\bx_{i'})\right)\Bigg]\\
&=\frac{(|B|-1)(|J| - 1)}{|B||J|N_{\mathcal{L}}^2}\mathbb{E}_{i,i',j,j'}\Bigg[\left(\mathcal{L}u_\theta(\bx_i) - R(\bx_i)\right)\left(\mathcal{L}u_\theta(\bx_{i'}) - R(\bx_{i'})\right)\left(\frac{\partial}{\partial \theta}\mathcal{L}_ju_\theta(\bx_i)\right)\left(\frac{\partial}{\partial \theta}\mathcal{L}_{j'}u_\theta(\bx_{i'})\right)\Bigg]\\
&=\frac{(|B|-1)(|J| - 1)}{|B||J|N_{\mathcal{L}}^2}\left(\mathbb{E}_{i,j}\Bigg[\left(\mathcal{L}u_\theta(\bx_i) - R(\bx_i)\right)\left(\mathcal{L}u_\theta(\bx_{i'}) - R(\bx_{i'})\right)\Bigg]\right)^2\\
&=\frac{(|B|-1)(|J| - 1)}{|B||J|}g(\theta)^2
\end{aligned}
\end{equation}
Taking together
\begin{equation}
\begin{aligned}
\mathbb{V}_{B, J}[g_{B, J}(\theta)] &= \mathbb{E}_{B, J}[g_{B, J}(\theta)^2] - g(\theta)^2\\
&=\frac{C_1 + C_2(|B|-1) + C_3(|J| - 1) + (1-|B|-|J|)g(\theta)^2}{|B||J|}\\
&= \frac{C_4|J| + C_5|B| + C_6}{|B||J|}.
\end{aligned}
\end{equation}

\end{proof}

\subsection{Proof of Lemma \ref{lemma:nn_derivative_bound}}\label{proof:nn_derivative_bound}
\begin{lemma*} (Lemma \ref{lemma:nn_derivative_bound} in the Main Text, Bounding the Neural Network Derivatives) Consider the neural network defined as Definition \ref{def:DNN} with parameters $\theta$, depth $L$, and width $h$, then the neural network's derivatives can be bounded as follows.
\begin{align}
\left\|\operatorname{vec}\left(\frac{\partial^n}{\partial\bx^n}u_\theta(\bx)\right)\right\| &\leq (n-1)!d^{n-1}(L-1)^{n-1}M(L) \prod_{l=1}^{L-1} M(l)^n,\\
\left\|\operatorname{vec}\left(\frac{\partial}{\partial \theta}\frac{\partial^n}{\partial\bx^n}u_\theta(\bx)\right)\right\|&\leq h^2n!d^{n}(L-1)^{n}M(L) \prod_{l=1}^{L-1} M(l)^{n+1} \max\left\{\Vert\bx\Vert, 1\right\},
\end{align}
where $M(l) = \max\left\{\Vert W_l \Vert, 1\right\}$, the norms are all vector or matrix 2 norms and $\operatorname{vec}$ denotes vectorization of the high-order derivative tensor, and $h$ is the maximal width of the neural network, i.e., $ h = \max(m_L, \cdots, m_0)$.
\end{lemma*}
\begin{proof}
Recall that the PINN's neural network structure in Definition \ref{def:DNN}:
\begin{equation*}
    u_\theta(\bx)=W_L \sigma (W_{L-1} \sigma(\cdots \sigma(W_1\bx)\cdots ).
\end{equation*}
After taking the derivative with respect to $\bx$, there is only one $\mathbb{R}^d$ vector term:
\begin{equation}
\frac{\partial u_{\theta}(\bx)}{\partial \bx} = W_L \cdot \Phi_{L-1}^{(1)}(\bx) W_{L-1} \cdot \dots \cdot \Phi_1^{(1)}(\bx) W_1 \in \mathbb{R}^{d},
\end{equation}
where
$\Phi_l^{(n)}(\bx) = \text{diag}[\sigma^{(n)}(W_l\sigma(W_{l-1}\sigma(\cdots\sigma(W_1\bx)\cdots)] \in \mathbb{R}^{m_l\times m_l}$ and $\sigma^{(n)}$ denotes the activation function's $n$th-order derivative which is also a pointwise function.
Hence, the first-order derivative's two norm has an upper bound of $\prod_{l=1}^L \Vert M(l)\Vert$ due to bounded activation function in Assumption \ref{assumption:activation}, i.e., $\Vert \Phi_l^{(n)}(\bx) \Vert \leq 1$. More specifically, since $\Vert AB\Vert \leq \Vert A \Vert\Vert B \Vert$ for all matrices $A, B$, we have
\begin{align}
\left\|\frac{\partial u_{\theta}(\bx)}{\partial \bx}\right\| \leq \prod_{l=1}^L \Vert W_l \Vert  \prod_{l=1}^{L-1} \left\| \Phi_{l}^{(1)}(\bx) \right\| \leq \prod_{l=1}^L \Vert W_l \Vert \leq\prod_{l=1}^L M(l).
\end{align}

After taking the second derivative, due to the nested structure of the neural network, the result will be the sum of $L-1$ parts induced by the derivative of each $\Phi_l^{(1)}(\bx), 1 \leq l \leq L-1$ with respect to $\bx$. Each of the $L-1$ parts is a $\mathbb{R}^{d \times d}$ matrix, and we further treat each row of them as one $\mathbb{R}^d$ vector term. Hence, there are $d(L-1)$ number of $\mathbb{R}^d$ vector terms in total:
\begin{equation}
\frac{\partial^2 u_{\theta}(\bx)}{\partial \bx^2} =\left\{\sum_{l=1}^{L-1} 
(W_L\Phi^{(1)}_{L-1}(\bx)\cdots W_{l+1})
\text{diag}(\Phi^{(2)}_l(\bx)W_l\cdots\Phi^{(2)}_1(\bx)(W_1)_{:,j})
(W_l\cdots \Phi_1^{(1)}(\bx)W_1)\right\}_{1 \leq j \leq d}
\end{equation}
where we call each 
$(W_L\Phi^{(1)}_{L-1}(\bx)\cdots W_{l+1})
\text{diag}(\Phi^{(2)}_l(\bx)W_l\cdots\Phi^{(2)}_1(\bx)(W_1)_{:,j})
(W_l\cdots \Phi_1^{(1)}(\bx)W_1) \in \mathbb{R}^d$
as a $\mathbb{R}^d$ vector term.
Each of these $d(L-1)$ number of $\mathbb{R}^d$ vector terms has an upper bound of $M_L \prod_{l=1}^{L-1} M_l^2$:
\begin{align}
&\quad\left\|(W_L\Phi^{(1)}_{L-1}(\bx)\cdots W_{l+1})
\text{diag}(\Phi^{(2)}_l(\bx)W_l\cdots\Phi^{(2)}_1(\bx)(W_1)_{:,j})
(W_l\cdots \Phi_1^{(1)}(\bx)W_1)\right\|\\
&\leq \left(\prod_{k=1}^L \Vert W_k \Vert\right)\left\|\Phi^{(2)}_l(\bx)W_l\cdots\Phi^{(2)}_1(\bx)(W_1)_{:,j}\right\| \leq \left(\prod_{k=1}^L \Vert W_k \Vert\right)\left(\prod_{k=1}^l \Vert W_l \Vert\right) \leq M_L \prod_{l=1}^{L-1} M_l^2.
\end{align}
In sum, for the second-order derivative, we have the lemma holds for the second-order case:
\begin{align}
\left\|\operatorname{vec}\left(\frac{\partial^2 u_{\theta}(\bx)}{\partial \bx^2}\right)\right\| \leq d(L-1)M_L \prod_{l=1}^{L-1} M_l^2.
\end{align}

Next, we try to discover the induction rule.
For the $n$th order derivative with $n \geq 1$, it is the sum and the concatenation of $(n-1)!d^{n-1}(L-1)^{n-1}$ number of $\mathbb{R}^d$ vector terms (0! = 1), where each of them is bounded by $M(L) \prod_{l=1}^{L-1} M(l)^n$. Furthermore, each of the terms contains the product of at most $n(L-1)$ matrix-valued function depending on $\bx$, which are $\{\Phi_l^{(1)}(\bx)\}_{l=1}^{L-1},\{\Phi_l^{(2)}(\bx)\}_{l=1}^{L-1}, \cdots \{\Phi_l^{(n)}(\bx)\}_{l=1}^{L-1}$. They are crucial in the product rule of derivatives.

For instance, in the first-order derivative where $n=1$, there is only one $\mathbb{R}^d$ term, bounded by $\prod_{l=1}^L M(l)$. This term also contains the product of $L-1$ matrix-valued function depending on $\bx$, which are $\{\Phi_l^{(1)}(\bx)\}_{l=1}^{L-1}$.

For instance, in the second-order case where $n=2$, the result is the sum and the concatenation of $d(L-1)$ terms, which are $(W_L\Phi^{(1)}_{L-1}(\bx)\cdots W_{l+1})
\text{diag}(\Phi^{(2)}_l(\bx)W_l\cdots\Phi^{(2)}_1(\bx)(W_1)_{:,j})
(W_l\cdots \Phi_1^{(1)}(\bx)W_1) \in \mathbb{R}^d$. Each of them can be bounded by $M(L)\prod_{l=1}^{L-1} M(l)^2$. These terms also contain the product of at most $2(L-1)$ matrix-valued function depending on $\bx$, which are $\{\Phi_l^{(1)}(\bx)\}_{l=1}^{L-1}$
and $\{\Phi_l^{(2)}(\bx)\}_{l=1}^{L-1}$.

Using the principle of induction, after taking the $n$th derivative, there will be $(n-1)!d^{n-1}(L-1)^{n-1}$ number of $\mathbb{R}^d$ vector terms, and each term has an upper bound of $M_L \prod_{l=1}^{L-1} M_l^n$. Each term contains at most $n(L-1)$ matrix-valued function depending on $\bx$. The key to the proof is that
\begin{align}
\frac{\partial \phi_l^{(n)}(\bx)}{\partial \bx} = \phi_l^{(n+1)}(\bx)W_l \cdots 
\phi_1^{(n+1)}(\bx) W_1, \quad \left\|\frac{\partial \phi_l^{(n)}(\bx)}{\partial \bx}\right\| \leq \prod_{k=1}^l \Vert W_k \Vert \leq \prod_{l=1}^{L-1} M(l), \quad \forall l \leq L-1,
\end{align}
where $\phi_l^{(n)}(\bx) = \sigma^{(n)}(W_l\sigma(W_{l-1}\sigma(\cdots\sigma(W_1\bx)\cdots)$ and $\Phi_l^{(n)}(\bx) = \text{diag}\left(\phi_l^{(n)}(\bx)\right)$.
In other words, taking derivative to $\Phi_l^{(n)}(\bx) = \text{diag}\left(\phi_l^{(n)}(\bx)\right)$ with respect to $\bx$ will generate $(l-1)$ more matrix-valued functions depending on $\bx$, which are $\left\{\Phi^{(n+1)}_k(\bx)\right\}_{k=1}^{l-1}$. 
The term's upper bound is multiplied/enlarged by at most $\prod_{l=1}^{L-1} M(l)$. 

Consequently, for the case of $(n+1)$th order derivative, it is the sum and the concatenation of at most
$(n-1)!d^{n-1}(L-1)^{n-1} \cdot n(L-1) \cdot d = n!d^n(L-1)^n$ number of $\mathbb{R}^d$ vector terms, where $(n-1)!d^{n-1}(L-1)^{n-1}$ is the number of $\mathbb{R}^d$ vector terms in the $n$th order derivative, $n(L-1)$ is the number of $\bx$-dependent matrix-valued functions in each term which is crucial in the product rule of derivatives, and $d$ is due to the derivative tensor size. Each of these terms is bounded by $M(L) \cdot \prod_{i=1}^{L-1} M(l)^n \cdot \prod_{i=1}^{L-1} M(l) = M(L) \cdot \prod_{i=1}^{L-1} M(l)^{n+1}$. Each of these terms contains the product of at most $(n+1)(L-1)$ matrix-valued functions depending on $\bx$, which are $\{\Phi_l^{(1)}(\bx)\}_{l=1}^{L-1},\{\Phi_l^{(2)}(\bx)\}_{l=1}^{L-1}, \cdots \{\Phi_l^{(n+1)}(\bx)\}_{l=1}^{L-1}$.
Hence, our induction assumption holds.
Also, the norm of $(n+1)$th order derivative tensor will be upper bounded by
\begin{align}
&\quad\left\|\operatorname{vec}\left(\frac{\partial^{n+1}}{\partial\bx^{n+1}}u_\theta(\bx)\right)\right\|\\
&\leq (n-1)!d^{n-1}(L-1)^{n-1} \cdot nd(L-1) \cdot M(L) \prod_{l=1}^{L-1} M(l)^{n} \cdot \prod_{l=1}^{L-1} M(l) \\
&= n!d^n(L-1)^n M(L) \prod_{l=1}^{L-1} M(l)^{n+1}.
\end{align}

If we further take the derivative with respect to $W_l$, the same logic and induction still apply. 
Concretely, we consider the derivatives with respect to the model parameter $\theta$ first. Since $\theta = \{W_l\}_{l=1}^L$, we consider $W_l$
\begin{align}
\frac{\partial u_\theta(\bx)}{\partial W_l} &=\left[W_L \Phi_{L-1}^{(1)}(\bx) \cdots W_{l}\Phi_{l}^{(1)}(\bx)\right]^\mathrm{T} \cdot  \sigma\left(W_{l-1} \cdots \sigma(W_1\bx) \cdots\right) \in \mathbb{R}^{m_l \times m_{l-1}}, \quad l > 1.\\
\frac{\partial u_\theta(\bx)}{\partial W_1} &=\left[W_L \Phi_{L-1}^{(1)}(\bx) \cdots W_{1}\Phi_{1}^{(1)}(\bx)\right]^\mathrm{T} \cdot  \bx \in \mathbb{R}^{m_1 \times d}
\end{align}
For all layer index $l$, $\frac{\partial u_\theta(\bx)}{\partial W_l}$ is at most a $\mathbb{R}^{h^2}$ vector, whose each entry is bounded by $\prod_{l=1}^{L} M(l)$ for $l \neq 1$. For $l = 1$, it is bounded by $\prod_{l=1}^{L} M(l) \cdot \Vert \bx \Vert$. Furthermore, each of the terms contains the product of at most $(L-1)$ matrix-valued function depending on $\bx$, which are $\{\Phi_l^{(1)}(\bx)\}_{l=1}^{L-1}$. They are crucial in the product rule of derivatives. Hence, our previous induction still applies to each of the $h^2$ entries, and the final bound will be
$h^2n!d^{n}(L-1)^{n}M(L) \prod_{l=1}^{L-1} M(l)^{n+1} \max\left\{\Vert\bx\Vert, 1\right\}$.
Intuitively, $\max\left\{\Vert\bx\Vert, 1\right\}$ is due to the derivative concerning $W_1$, $h^2$ is due to the size of model parameters, $n+1$ is due to the $n$th order derivatives for $\bx$ and another derivative for $\theta$.

\end{proof}

\subsection{Proof of Theorem \ref{thm:convergence}} \label{appendix:convergence}

\begin{theorem*} (Theorem \ref{thm:convergence} in the Main Text, Convergence of SDGD under Bounded Gradient Variance) Assume Assumptions \ref{assumption:activation} and \ref{assumption:operator} hold, fix some tolerance level $\delta >0$, suppose that the SGD trajectories given in equations (\ref{eq:SGD1}, \ref{eq:SGD2}) are bounded, i.e., $\Vert W_l^n \Vert\leq M(l)$ for all epoch $n$ where the collection of all $\{W_l^n\}_{l=1}^L$ is $\theta^n$, and that the regular minimizer of the PINN loss is $\theta^*$ as defined in Definition \ref{def:reg_min}, and that the SGD step size follows the form $\gamma_n = \gamma / (n + m)^p$ for some $p \in (1/2,1]$ and large enough $\gamma, m > 0$, then
\begin{enumerate}
\item There exist neighborhoods $\mathcal{U}$ and $\mathcal{U}_1$ of $\theta^*$ such that, if $\theta^1 \in \mathcal{U}_1$, the event
\begin{equation}
E_\infty = \left\{\theta^n \in \mathcal{U} \ \text{for all } n \in \mathbb{N}\right\}
\end{equation}
occurs with probability at least $1 - \delta$, i.e., $\mathbb{P}(E_\infty | \theta^1 \in\mathcal{U}_1) \geq 1 - \delta$.
\item Conditioned on $E_\infty$, we have
\begin{equation}
\mathbb{E}[\Vert\theta^n - \theta^*\Vert^2|E_\infty] \leq \mathcal{O}(1/n^p).
\end{equation}
\end{enumerate}
\end{theorem*}

\begin{proof}
Since we assume the SGD trajectories in equations (\ref{eq:SGD1}, \ref{eq:SGD2}) are bounded, we can denote such bound as $M(l) := \max\left\{\max_n\left\{\Vert W_l^n \Vert\right\}, 1\right\} < \infty$. 

By Lemma \ref{lemma:nn_derivative_bound}, the norm of high-order derivatives of the neural network throughout the optimization will fall into a compact domain. 

Due to Assumption \ref{assumption:operator}, the values generated by the PDE operator acting on the neural network and its derivative with respect to model parameters $\theta$ (the gradient for optimization) will also fall into a bounded domain throughout the optimization, whose closure is compact.

Since we bound the gradient produced by physics-informed loss functions, we can further bound the gradient variance during PINN training using SDGD.

For the stochastic gradient produced by Algorithm \ref{algo:1},
\begin{align}
\mathbb{V}_{B,J}(g_{B, J}(\theta)) &\leq \frac{1}{|B||J|N_rN^2_{\mathcal{L}}}\sum_{i=1}^{N_r}\sum_{j=1}^{N_{\mathcal{L}}}\Vert g_{i,j}(\theta) - g(\theta)\Vert^2 \leq \frac{2}{|B||J|N_{\mathcal{L}}}\max_{i,j} \Vert g_{i,j}(\theta) \Vert^2,
\end{align}
where
$
g_{i,j}(\theta) := \left(\mathcal{L}u_\theta(\bx_i) - R(\bx_i)\right)\left(\frac{\partial}{\partial \theta}\mathcal{L}_ju_\theta(\bx_i)\right).
$

For the stochastic gradient produced by Algorithm \ref{algo:2},
\begin{equation}
\begin{aligned}
g_{B, J, K}(\theta) &= \frac{1}{|B||J|N_{\mathcal{L}}}\sum_{i \in B}\left(\frac{N_{\mathcal{L}}}{|K|}\left(\sum_{k \in K}\mathcal{L}_ku_\theta(\bx_i)\right) - R(\bx_i)\right)\left(\sum_{j \in J}\frac{\partial}{\partial \theta}\mathcal{L}_ju_\theta(\bx_i)\right)\\
&=\frac{1}{|B||J||K|}\sum_{i \in B}\sum_{j \in J}\sum_{k \in K}\left(\mathcal{L}_ku_\theta(\bx_i) - \frac{R(\bx_i)}{N_{\mathcal{L}}}\right)\left(\frac{\partial}{\partial \theta}\mathcal{L}_ju_\theta(\bx_i)\right),
\end{aligned}
\end{equation}
\begin{align}
\mathbb{V}_{B,J,K}(g_{B, J, K}(\theta)) &\leq \frac{1}{|B||J||K|N_rN^2_{\mathcal{L}}}\sum_{i=1}^{N_r}\sum_{j=1}^{N_{\mathcal{L}}}\sum_{k=1}^{N_{\mathcal{L}}}(g_{i,j,k}(\theta) - g(\theta))^2 \leq \frac{2}{|B||J||K|}\max_{i,j,k} \Vert g_{i,j,k}(\theta) \Vert^2,
\end{align}
where
$
g_{i,j,k}(\theta) := \left(\mathcal{L}_ku_\theta(\bx_i) - \frac{R(\bx_i)}{N_{\mathcal{L}}}\right)\left(\frac{\partial}{\partial \theta}\mathcal{L}_ju_\theta(\bx_i)\right).
$

Consequently,
\begin{equation}
\Vert g_{i,j}(\theta)\Vert\leq \left\|\mathcal{L}u_\theta(\bx_i) - R(\bx_i)\right\|\left\|\frac{\partial}{\partial \theta}\mathcal{L}_ju_\theta(\bx_i)\right\|\leq \left(\max_{i, j}|\mathcal{L}_ju_\theta(\bx_i)| + R\right)\max_{i, j}\left\|\frac{\partial}{\partial \theta}\mathcal{L}_ju_\theta(\bx_i)\right\|,
\end{equation}
\begin{equation}
\Vert g_{i,j,k}(\theta)\Vert\leq \left\|\mathcal{L}_ku_\theta(\bx_i) - \frac{R(\bx_i)}{N_{\mathcal{L}}}\right\|\left\|\frac{\partial}{\partial \theta}\mathcal{L}_ju_\theta(\bx_i)\right\|\leq \left(\max_{i, j}|\mathcal{L}_ju_\theta(\bx_i)| + R\right)\max_{i, j}\left\|\frac{\partial}{\partial \theta}\mathcal{L}_ju_\theta(\bx_i)\right\|.
\end{equation}
Since $|\mathcal{L}_ju_\theta(\bx_i)|$ and $\left\|\frac{\partial}{\partial \theta}\mathcal{L}_ju_\theta(\bx_i)\right\|$ can be bounded throughout the optimization process thanks to Lemma \ref{lemma:nn_derivative_bound} and Assumption \ref{assumption:operator}, the gradient variance during optimization is universally bounded and finite, i.e., the gradient variance bounds are agnostic of the epoch $n$.

Thus, we complete the proof by applying Theorem 4 in Mertikopoulos et al. \cite{mertikopoulos2020almost}.
\end{proof}

\newpage
\bibliographystyle{plain}
\bibliography{cas-refs.bib}

\end{document}